\documentclass{article}

\usepackage{mlmodern}
\usepackage[T1]{fontenc}

\usepackage{hyperref}
\usepackage{url}
\usepackage{mathtools}
\usepackage{amsmath, amssymb, amsfonts, bm, amsthm, bbm,xfrac}
\usepackage{thmtools,thm-restate}
\usepackage{xcolor}
\usepackage{cleveref}
\usepackage{enumitem}
\usepackage{authblk}
\usepackage[margin=1in]{geometry}
\usepackage{natbib}
\usepackage{placeins}
\setcitestyle{authoryear,round,citesep={;},aysep={,},yysep={;}}

\hypersetup{
    colorlinks = true,
    linkcolor = {blue},
    citecolor = {blue},
    urlcolor = {blue}
}

\numberwithin{equation}{section}
\theoremstyle{definition}
\newtheorem{theorem}{Theorem}
\newtheorem{proposition}{Proposition}
\newtheorem{lemma}{Lemma}[section]
\newtheorem{corollary}{Corollary}
\newtheorem{definition}{Definition}
\newtheorem{assumption}{Assumption}[section]
\newtheorem{example}{Example}
\newtheorem{remark}{Remark}

\newcommand{\RR}[0]{\mathbb{R}}
\newcommand{\NN}[0]{\mathbb{N}}
\newcommand{\ZZ}[0]{\mathbb{Z}}
\newcommand{\EE}[0]{\mathbb{E}}

\newcommand{\deff}[0]{d_\mathrm{eff}}
\newcommand{\KK}[0]{\mathbf{K}}

\newcommand{\PP}[0]{\mathbf{P}}

\newcommand{\dd}[0]{\rm{d}}
\newcommand{\Tr}[0]{{\rm Tr}}
\newcommand{\He}[0]{{\rm He}}

\author[1]{Arie Wortsman}
\author[1]{Bruno Loureiro}

\affil[1]{\small Departement d'Informatique, \'Ecole Normale Sup\'erieure, PSL \& CNRS}

\title{Kernel ridge regression under power-law data: \\ spectrum and generalization}

\date{}

\begin{document}

\maketitle

\begin{abstract}
In this work, we investigate high-dimensional kernel ridge regression (KRR) on i.i.d. Gaussian data with anisotropic power-law covariance. This setting differs fundamentally from the classical source \& capacity conditions for KRR, where power-law assumptions are typically imposed on the kernel eigenspectrum itself. Our contributions are twofold. First, we derive an explicit characterization of the kernel spectrum for polynomial inner-product kernels, giving a precise description of how the kernel eigenspectrum inherits the data decay. Second, we provide an asymptotic analysis of the excess risk in the high-dimensional regime for a particular kernel with this spectral behavior, showing that the sample complexity is governed by the effective dimension of the data rather than the ambient dimension. These results establish a fundamental advantage of learning with power-law anisotropic data over isotropic data. To our knowledge, this is the first rigorous treatment of non-linear KRR under power-law data.
\end{abstract}

\section{Introduction}
Consider a supervised learning problem where training data $(x_1, y_1), \dots, (x_n, y_n) \in \RR^{d} \times \RR$ is sampled i.i.d. from a joint probability distribution over $\mathbb{R}^{d}\times \mathbb{R}$ with density $\nu$. In this manuscript are interested in the problem of kernel ridge regression (KRR):
\begin{align}
\label{eq:def:krr}
    \hat{f}_{\lambda} : = \underset{f \in \mathcal{H} }{\mathrm{arg}\min}\left \{ \sum_{i=1}^{n} \left(y_{i} - f(x_i)\right)^{2} + \lambda \| f\|_{\mathcal{H}}\right \},
\end{align}
where $k$ is a positive-definite kernel associated with the \emph{reproducing kernel Hilbert space} (RKHS) $\mathcal{H}$, and $\lambda>0$ is the $\ell_{2}$-penalty strength. Throughout this manuscript, we assume $k$ is universal and trace-class.

Although Kernel Ridge Regression (KRR) has long been a central topic in classical machine learning \citep{W64,scholkopf2002learning}, it has recently attracted renewed interest owing to its connections with neural networks, both at initialization \citep{williams1996computing,lee2018deep} and in the lazy training regime \citep{JGH18,COB19}.

Our main focus in the following will be on the study of the generalization properties of the minimizer of \cref{eq:def:krr}, as quantified by the excess population risk:
\begin{align}
\label{eq:risk}
    R(\hat{f}_{\lambda}) = \mathbb{E}_{x\sim \nu_{x}}\left[\left(\hat{f}_{\lambda}(x)-f_{\star}(x)\right)^{2}\right],
\end{align}
where the expectation is taken over an independent sample from the covariates marginal distribution $x\sim \nu_{x}$, and $f_{\star}(x)=\mathbb{E}[y|x]$ is the Bayes predictor. We will further assume that $f_{\star}\in L^{2}(\nu_{x})$ and the noise $\varepsilon_{i}=y_{i}-f_{\star}(x_{i})$ is zero mean and has finite variance $\mathbb{E}[\varepsilon_{i}^{2}]=\sigma^{2}<\infty$.\footnote{A well-known result is that, under mild conditions on $K$, the kernel ridge regressor is a universal approximator on $L^{2}(\nu_{x})$ \citep{micchelli2006universal}. Therefore, we can assume without loss of generality that $f_{\star}\in L^{2}(\nu_{x})$, with any other component effectively behaving as irreducible noise.}

The generalization properties of \cref{eq:def:krr} have been studied in the learning theory literature under different assumptions. In particular, two settings have received significant attention. The first, known under the umbrella of \emph{source and capacity} conditions, considers a family of tasks parametrized by the relative complexity of $\mathcal{H}$ with respect to $L^{2}(\nu_{x})$, as characterized by the spectral decomposition of the kernel. More precisely, consider the \emph{kernel operator} $T:L^{2}(\nu_{x})\to \mathcal{H}$ defined as
\begin{equation}
    T(f) = \int_{\RR^{d}} k(x,x') f(x') \nu_{x}(\dd x').
    \label{eq:integral_operator}
\end{equation}
Since this is a self-adjoint operator, it admits a diagonalization in $L^{2}(\nu_{x})$ \citep{cucker2002mathematical}. Let $\lambda_{m}\geq 0$ denote its eigenvalues, ordered non-increasingly, and $e_{m}$ the corresponding eigenfunctions. Because $k$ is trace-class, we have $\Tr~T = \sum_{m\geq 0}\lambda_{m}<\infty$, and the effective “size’’ of $\mathcal{H}\subset L^{2}(\nu_{x})$ is governed by the rate of decay of the eigenvalues. Similarly, the complexity of $f_{\star}\in L^{2}(\nu_{x})$ is quantified by the magnitude of $||T^{1/2}f_{\star}||_{\mathcal{H}}$. The source and capacity conditions formalize these notions by assuming a power-law decay for these quantities:
\begin{itemize}
    \item\textbf{Capacity:} There exists a $\alpha> 1$ such that ${\Tr~T^{\alpha}=\sum_{m\geq 0}\lambda^{\alpha}_{m}<\infty}$.
    \item\textbf{Source:} There exists a $r\geq 0$ such that ${||T^{1/2-r}f_{\star}||_{\mathcal{H}}<\infty}$.
\end{itemize}
The excess risk rates for KRR under these conditions have been extensively analyzed in the kernel literature \citep{CdV07,bach2017equivalence,richards2021asymptotics}, revealing a rich phenomenology with cross-overs between different decay and plateau regimes \citep{cui2021generalization,defilippis2024dimension} reminiscent of the empirically observed \emph{neural scaling laws} \citep{brown2020language,kaplan2020scaling,hoffmann2022empirical}. This parallel has sparked renewed interest in these conditions, with many recent works exploring closely related settings as theoretical proxies for neural scaling laws \citep{bahri2024explaining,maloney2022solvable,atanasov2024scaling,bordelon2024dynamical,paquette2024four}.

A complementary line of work instead considers explicit kernel functions and data distributions for which the connection between data and feature space is mathematically tractable. Results of this type, however, are rare, as they rely on an explicit diagonalization of the kernel in $L^{2}(\nu_{x})$, which is generally a very challenging problem. Two notable exceptions are: (i) low-dimensional problems, where diagonalizing the integral operator in \cref{eq:integral_operator} can be reduced to solving a differential equation \citep{tomasini2022failure}; and (ii) dot-product kernels with isotropic data (e.g.\ $x\sim\mathcal{N}(0,I_{d})$ or $x\sim{\rm Unif}(\mathbb{S}^{d-1})$), where the eigenfunctions are given by harmonic polynomials \citep{GMMM20,ghorbani2021linearized,MMM22}. A key consequence of the latter results is that, since $\lambda_{m}=\Theta(d^{-m})$, learning high-frequency components of the target function $\langle f_{\star},e_{m}\rangle$ requires increasingly fine spectral resolution, leading to a high-dimensional sample complexity bottleneck for KRR of $n=\Theta(d^{m})$, analogous to polynomial ridge regression \citep{MMM22}. 

Our main goal in this paper is to go beyond the isotropic high-dimensional setting, addressing the following question:
\begin{center}
\textit{How does structure in the covariates impact the generalization properties of kernel methods?}
\end{center}
Motivated by the ubiquity of power-law structure in signal processing \citep{simoncelli2001natural, mallat2002theory}, we consider the setting where the covariates follow an anisotropic Gaussian distribution $x\sim\mathcal{N}(0,\Sigma)$, with $\Sigma\in\mathbb{R}^{d\times d}$ taken, without loss of generality, to be diagonal $\Sigma_{jk}=\sigma_{j}\delta_{jk}$ with a power-law spectrum:
\begin{align}
\sigma_{j}= C_{\alpha}(d) \cdot j^{-\alpha}, \quad 1\leq j \leq d
\label{eq:cov_eigenvalues}
\end{align}
where $\alpha\geq 0$ and $C_{\alpha}(d)$ is chosen such that $\Tr{\Sigma} = 1$. In particular, we denote the corresponding probability density function by $\gamma_{d}^{\alpha}$. Our \textbf{main contributions} are:
\begin{itemize}
    \item \textbf{Sharp Spectrum:} We establish an exact asymptotic characterization of the spectrum of polynomial dot-product kernels as $d\to\infty$, valid for all $\alpha \geq 0$. For $\alpha>1$, the kernel provably satisfies an asymptotic capacity condition with $\lambda_{m}=\Theta(m^{-\alpha})$, exactly mirroring the decay of the data covariance.
    \item \textbf{Excess risk structured data:} We derive an asymptotic characterization of the excess risk for a particular family of kernels with the above spectrum in the high-dimensional scaling regime where $\alpha\in[0,1)$. The analysis shows that the risk is governed by the \emph{effective dimension} of the data, which decreases with $\alpha$, thereby establishing a fundamental statistical advantage of power-law structure for KRR.
\end{itemize}
Finally, we provide numerical experiments to illustrate our theoretical results, as well as to show its relevance beyond the scope of the theory.

\subsection*{Further related works}
\paragraph{KRR with anisotropic data:} Anisotropy in the data distribution of KRR has been investigated in different contexts. \cite{LR20} investigated how $\Sigma$ impacts the generalization of KRR at the interpolation regime ($\lambda = 0$). \cite{DWY21} studied rotationally invariant kernels for anisotropic sub-Gausssian data in a high-dimensional setting.  \cite{MM22} studied KRR and Neural Networks with a structured covariance for spherical distributions. \cite{BEMWW24,MWSE23,wang2024nonlinear} studied KRR on data with a spiked covariance matrix. However, none of these works address the anisotropic power-law setting considered here.

\paragraph{Theory of scaling laws:} 
Scaling laws are a classical topic in the kernel literature, extensively studied under the framework of source and capacity conditions. In particular, several works have characterised the scaling of the excess risk for KRR \citep{CdV07,bach2017equivalence,cui2021generalization}, random features \citep{rudi2017generalization,defilippis2024dimension}, and (S)GD \citep{yao2007early,ying2008online,carratino2018learning,pillaud2018statistical}. Distinct from our approach, these analyses assume power-law structure in feature space. More recently, \citep{bahri2024explaining,maloney2022solvable,atanasov2024scaling,bordelon2024dynamical,paquette2024four,lin2024scaling,KB25} examined the scaling behaviour of linear models trained on anisotropic data, under both ridge regression and (S)GD. Although these works introduce power-law structure in the inputs, linearity of the model directly implies a power-law structure in the features. Beyond linear settings, \citep{ren2025emergence,arous2025learning,defilippis2025scaling} studied scaling laws for two-layer neural networks in teacher–student setups, where the teacher weights follow a power-law decay and the data are isotropic Gaussian. In these models, non-linearity arises in the features, but anisotropy is only present in the target weights. To our knowledge, our work is the first to address the problem of anisotropic power-law data with non-linear features.

\subsection*{Notation} 
We denote $\gamma_{d}^{\alpha}$ as the gaussian measure in \cref{eq:cov_eigenvalues}. For an integer $m \in \NN$, we denote the set $[m]:=\{ 1, \dots, m\}$. We denote multi-indices in $\ZZ^{d}_{\geq 0}$ by Greek letters. Given a multi-index $\beta \in \ZZ^{d}_{\geq 0}$, we denote $|\beta| = \beta_1 + \dots + \beta_d$. We will sometimes denote $\beta!: = \beta_1! \dots \beta_d! $, which should not be confused with $|\beta|!$, which is the classical factorial for integer numbers. Following this notation, we will sometimes denote binomial coefficients $\binom{|\beta|}{\beta_1, \dots, \beta_d}: = \frac{|\beta|!}{\beta_1! \cdots \beta_d !}$ as $\binom{|\beta|}{\beta}$. For a vector $z \in \RR^{d}$ and a multi-index $\beta \in \ZZ^{d}_{\geq 0}$, we will denote $z^{\beta} := z_1^{\beta_1} \cdots z_d^{\beta_d}$. For a set $S$, it's cardinality is denoted by $|S|$. For a kernel $k: \RR^{d} \times \RR^{d} \to \RR$, it's Hilbert-Schmidt norm is denoted by $\|k\|_{\rm HS} := \left ( \int_{\RR^{d} \times \RR^{d}} k(x,x')^2 \mu(\dd x) \mu(\dd x')\right)^{\sfrac{1}{2}}$.

\section{Main results}
\label{sec:mainres}

In this section we discuss our two main results, concerning the characterization of the kernel spectrum and the consequences for the excess risk in the anisotropic high-dimensional regime.

While in the isotropic setting the natural scale in the problem is given by the data dimension, for strongly anisotropic data, this is played by the notion of \emph{effective dimension}. 
\begin{definition}[Effective Dimension] Let $\Sigma \in \RR^{d \times d}$ denote a positive semi-definite matrix with eigenvalues $\sigma_1 \geq \sigma_d \geq \dots \geq \sigma_d>0$. Define the following two notions of effective dimensionality: 
\begin{align*}
r_0(\Sigma) = \dfrac{\sum_{i=1}^{d} \sigma_i}{\sigma_1}, \text{ and } R_0(\Sigma) = \dfrac{\left (\sum_{i=1}^{d} \sigma_i \right )^{2}}{ \sum_{i=1}^{d} \sigma_i^2}.
\end{align*}
\label{def:effective_dim}
\end{definition}
These are standard notions that naturally arise in the analysis of anisotropic problems, e.g.\ \citep{BLLT20,cheng2024dimension}, and they will play a central role in our proofs. In the power-law setting introduced in \cref{eq:cov_eigenvalues}, we have $C_{\alpha}=r_0(\Sigma)^{-1}$, and the effective dimensions exhibit the following asymptotic scaling as $d\to\infty$:
\begin{equation}
r_0(\Sigma) = \begin{cases}
    O(d^{1-\alpha}) & \text{ for } 0 \leq \alpha \leq 1 \\
    \log(d), & \text{ for } \alpha = 1\\
    O(1), & \text{ for } \alpha >1, \end{cases}
\label{eq:asymptotics_effective dimension_1}
\end{equation}
while for $R_0(\Sigma)$:
\begin{equation}
    R_0(\Sigma) = \begin{cases}
    O(d) & \text{ for } 0 \leq \alpha \leq \frac{1}{2} \\
    O(d^{2-2\alpha}), & \text{ for } \frac{1}{2} < \alpha <1 \\
    O(1), & \text{for } \alpha >1. \end{cases}
    \label{eq:asymptotics_effective_dimension_2}
\end{equation}
\begin{remark}
Note that $R_0(\Sigma)$ exhibits a transition at $\alpha=\sfrac{1}{2}$, whereas $r_0(\Sigma)$ does not. In particular, this means for ${\alpha>\frac{1}{2}}$ the leading eigenvalues are significantly larger than the tail of the spectrum. To see this more concretely, we can see how this affect concentration inequalities. If we consider $x,x' \sim \gamma_d^{\alpha}$, then by Bernstein's Inequality we will have that $|\langle x,x\rangle |\sim \frac{\log(d)}{R_0(\Sigma)^{\frac{1}{2}}}$ with high probability. Then, when $\alpha < \frac{1}{2}$, this will be the standard asymptotic bound $|\langle x,x\rangle| \sim \frac{\log(d)}{\sqrt{d}}$, while for $\alpha > \frac{1}{2}$, this gives $|\langle x,x\rangle| \sim \frac{\log(d)}{\deff}$. This shows that for $\alpha < \frac{1}{2}$, each coordinate contributes to the behavior of the sum, while for $\alpha > \frac{1}{2}$, only the first few coordinates determine the order of the sum. 
\end{remark}

\subsection{Spectrum of an inner-product kernels}
\label{sec:spectrum}
\begin{figure}[t]
    \centering
    \includegraphics[width=0.8\linewidth]{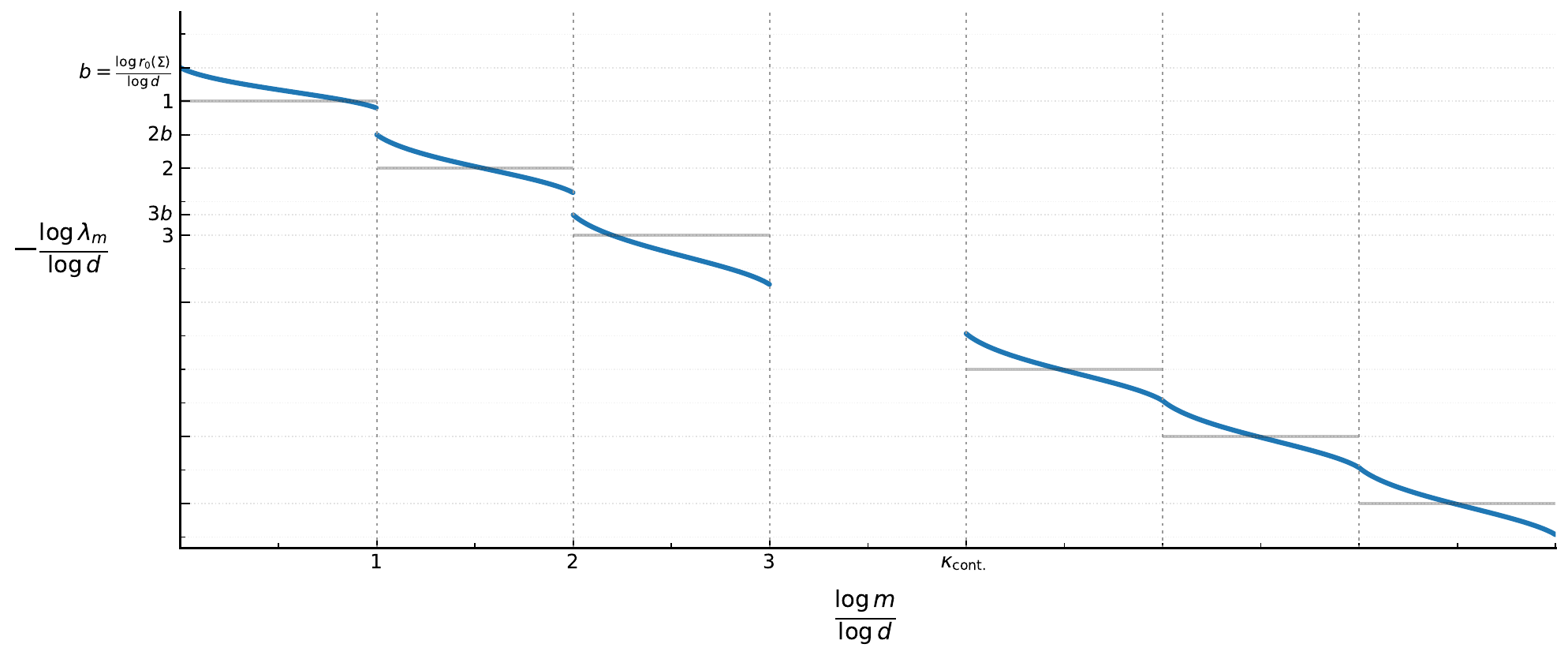}
    \caption{Illustration of the kernel spectrum for $\alpha \in [\frac{1}{\ell + 1},\frac{1}{\ell})$, for $\ell \in \NN$, from \cref{prop:spectrum_power_law_final}, shown in normalized log–log scale and highlighting both the \emph{spectral gap} and \emph{continuous} regions. The grey solid horizontal line corresponds to the isotropic case, where the degenerate eigenvalues are grouped into piecewise constant levels: at each level $m \geq 0$, there are $\Theta(d^{m})$ eigenvalues of magnitude $\Theta(d^{-m})$. By contrast, the black solid line depicts the anisotropic case with $\alpha \in (0,1)$, where the spectrum separates into two distinct regimes. In the \emph{spectral gap region}, on the left of the figure, levels $m \leq \kappa_{\rm cont.}= \ell$ contain $\Theta(d^{m})$ non-degenerate eigenvalues of order $\Theta(r_0(\Sigma)^{-m})$ and increasing steepness, with successive levels separated by spectral gaps of decreasing side, starting at multiples of $b=\sfrac{\log{r_0(\Sigma)}}{\log{d}}$. Beyond this, in the \emph{continuous region} $m \geq \kappa_{\rm cont.}$, the gaps disappear and the eigenvalues overlap across levels, yielding a continuous spectrum that becomes increasingly steep at each level $m$.}
    \label{fig:tikz_power_law_decay}
\end{figure}
Our starting point in this section is to characterize the spectrum of the kernel operator defined in \cref{eq:integral_operator} for anisotropic Gaussian data. This question is central, as the generalization error of KRR is tightly connected to the spectrum of the kernel (see, e.g.\ \cite{CLKZ23,MMM22}). Our focus will be in inner-product kernels of the form
\begin{equation}
    k(x,x') = h \left (\langle x, x' \rangle \right ),
\end{equation}
where $h\in \mathcal{C}^{\infty}$. 
\begin{assumption} The function $h(\cdot): \RR \to \RR$ is a $\mathcal{C}^{\infty}$ function, and it has a series expansion:    
\begin{equation}
    h(t) = \sum_{m \geq 0} h_m t^m,
    \end{equation}
    where $h_k \geq 0 $ for all $k \in \NN \cup \{ 0\}$. 
\label{Assumptions_h}
\end{assumption}
Inner-product kernels have been extensively studied since the pioneering work of \cite{EK10}, who derived a sharp asymptotic approximation for the kernel matrix under isotropic sub-Gaussian data in the proportional regime $n=\Theta(d)$. This analysis has since been extended in several directions, including different normalizations \citep{cheng2013spectrum,fan2019spectral}, random features and NTK kernels \citep{MMM22, fan2020spectra} and polynomial scaling regimes \citep{lu2025equivalence,PWZ24}. In contrast, the anisotropic sub-Gaussian setting considered here remains largely unexplored.

Our first result concern the behavior of the spectrum of truncated inner-product kernels for any diagonal covariance matrix $\Sigma_{jk}=\sigma_{j}\delta_{jk}$.
\begin{proposition}
    Let $\sigma_1, \dots, \sigma_d \in \RR_{+}$, and define the diagonal covariance matrix $\Sigma = \mathrm{diag}(\sigma_1, \dots, \sigma_d)$. Then the integral operator  $T_{\leq D}$ associated to the truncated kernel:
    \[
    k^{\leq D}(x,x') = \sum_{m = 0}^{D} h_{m} \langle x_{m}, x'_{m}\rangle^{m},
    \]
    has $\binom{d + D}{D}$ non-zero eigenvalues. Moreover, for each multi-index $\beta \in \ZZ^{d}_{\geq 0}$, with $|\beta|= \beta_1 + \dots +\beta_d \leq D$, there exists an eigenvalue $\lambda_{\beta}$ and explicit constants $C_1, C_2$ such that:  
    \[
    C_{1,\beta} \sigma_1^{\beta_d}  \cdots \sigma_d^{\beta_d} \leq  \lambda_{\beta}  \leq C_{2,\beta} \sigma_1^{\beta_1} \cdots \sigma_d^{\beta_d},
    \]
    with $C_{i,\beta}$ constant on $d$ for $i \in \{1,2\}$. 
    \label{prop:eigenvalues_operator}
\end{proposition}

\begin{proof}[Sketch of the Proof:] 
We begin by noting that Hermite polynomials are not the eigenfunctions of this kernel. However, since any polynomial of degree $\leq D$ can be written as a linear combination of Hermite polynomials of degree $\leq D$, we can always rewrite our kernel in this basis. To do so, note that given $i, j \in [n]$, we can write: 
\[
k^{\leq D}(x_i, x_j) = \Phi_{i}^{\top} \Phi_{j},
\]
where $\Phi_{i}, \Phi_j \in \RR^{{\binom{D + d}{D}}}$ are feature vectors with coordinates indexed by multi-indices, and with elements $ \Phi_{i,\beta} = \sqrt{\binom{|\beta|}{\beta} \sigma^{\beta}} z_i^{\beta}$, and $z_i = \Sigma^{-\sfrac{1}{2}}x_i$. Following \cite{LRZ20}, we can construct a change of basis matrix $\sigma$ that transforms Hermite features $ \Psi_{i,\beta} = \sqrt{\binom{|\beta|}{\beta} \sigma^{\beta}} He_{\beta}(z)$ into monomial features $\Phi_{i,\beta} = \sqrt{\binom{|\beta|}{\beta} \sigma^{\beta}} z_i^{\beta}$ linearly, that is: 
\begin{equation}
    \Phi_{i} = \Lambda \Psi_{i}. 
\end{equation}
The change-of-basis matrix $\Lambda$ has a few interesting properties. In particular, for the positive-definite truncated kernel $k^{\leq D}$, this matrix is upper triangular and $\max \{\| \Lambda \|_\mathrm{op}, \| \Lambda^{-1}\|_\mathrm{op} \} \leq C$, for a dimension-free matrix $C$. Hence, the kernel matrix $K \in \RR^{n \times n}$ can be written as: 
\begin{equation}
    K = \Lambda \Psi \Psi^{\top} \Lambda,
\end{equation}
where $\Psi = [\Psi_1, \dots, \Psi_n]^{\top}$. \Cref{prop:eigenvalues_operator} follows from relating the eigenvalues of the operator with the eigenvalues of the expectation of $K$ over the data,  $\EE [K]$ and noting that $\Lambda$ acts a similarity transform. We refer the reader to \cref{Appendix:Gramm_Schmidt} for a detailed proof.
\end{proof}

\begin{remark}
    The techniques in \cite{LRZ20} also allow the distribution of $x$ to be sub-gaussian with independent entries. We left the generalization of this results to a more general sub-gaussian setting for future work. 
\end{remark}

\Cref{prop:eigenvalues_operator} gives, up to constants, the spectrum of the truncated kernel $k^{\leq D}$. However, it does not give an order for the eigenvalues. 
\begin{remark}[Isotropic case]
In the isotropic case ($\alpha=0$), this result is closely related to \cite{GMMM20}, which showed that for data uniformly distributed on the sphere the eigenvalues separate into distinct levels, each corresponding to a different scale in $d$. Specifically, each level $m \in [D]$ consists of $O(d^{m})$ degenerate eigenvalues of order $\Theta(d^{-m})$. This is expected, since the isotropic Gaussian distribution and the uniform distribution on the sphere are known to be asymptotically equivalent.
\end{remark}

\Cref{prop:eigenvalues_operator} can be extended to regular inner-product kernels. Indeed, kernels satisfying \cref{Assumptions_h} can be accurately approximated by truncating their Taylor expansion at degree $D$, with an error term that decreases with $D$ and can be explicitly controlled. Since \cref{prop:eigenvalues_operator} holds for any $D>0$, the spectrum of $k(x,x') = h(\langle x, x'\rangle)$ can be approximated by that of $k^{\leq D}(x,x')$. By tracking the approximation error, one shows that $|k - k^{\leq D}|_{\rm HS}\to 0$ as $D\to\infty$. The following corollary then follows directly from the Hoffman–Wielandt inequality (Thm.~2.2 in \cite{KG00}).
\begin{corollary}
 Let $\sigma_1, \dots, \sigma_d \in \RR_{+}$, and define the diagonal covariance matrix $\Sigma = \mathrm{diag}(\sigma_1, \dots, \sigma_d)$. Then the eigenvalues of integral operator  of the kernel $k(x,x') = h( \langle x,x'\rangle)$ can be bounded above and below by quantities of the same form as Proposition~\ref{prop:eigenvalues_operator}, up to constants independent of $d$.
 \label{corollary:continous_kernel_eigenvalues}
\end{corollary}
For isotropic data, the behavior of such kernels implies that the spectrum of $k$ exhibits a new spectral gap at each successive kernel degree. This phenomenon, however, does not persist in the anisotropic case: once $\alpha>0$, only finitely many spectral gaps remain, and for $\alpha>\sfrac{1}{2}$ the spectrum becomes continuous. See \cref{fig:experiment_1} for an illustration.

Although \cref{prop:eigenvalues_operator} does not provide an ordering of the eigenvalues for general $\Sigma$, in the power-law setting with $\sigma_i \propto i^{-\alpha}$ for $\alpha \geq 0$ we can determine the order of the $k$-th largest eigenvalue for each $m \leq d^{k}$ when focusing on a specific polynomial.
\begin{corollary}
\label{cor:monomialkernel}
    Let $\alpha \geq 0$, and consider the power-law covariance matrix in \eqref{eq:cov_eigenvalues}. Fix $D \in \NN$, and let $k(x,x') = h\left (\langle x, x' \rangle \right )$ with $h(x) = x^{D}$. Then the associated kernel operator $T_{D}$ has $\binom{d-1+D}{d-1}$ eigenvalues, denoted by $\lambda_{m}$ for $m \in \bigl[\binom{d-1+D}{d-1}\bigr]$. Moreover, for each such eigenvalue there exist constants $C_{1},C_{2}>0$, depending only on $\alpha$ and $D$, such that:
    \[
     C_1 \dfrac{m^{-\alpha} \mathrm{poly} \log(d)}{r_0(\Sigma)^{D}} \leq \lambda_{m} \leq C_2 \dfrac{m^{-\alpha} \mathrm{poly} \log(d)}{r_0(\Sigma)^{D}}.
    \]
    \label{cor:power_law_decay_monomial}
\end{corollary}

\begin{proof}[Sketch of the Proof:]
The classical approach to estimating the order of the eigenvalues is to approximate the number of eigenvalues lying in a set of the form $\{\lambda_{m} : \lambda_m \geq \varepsilon\}$, and then approximate this count by the volume of the corresponding polytope. In our setting, however, we can exploit the special structure of the eigenvalues --- specifically, their explicit dependence on integers --- to reformulate the problem. The cardinality of the polytope can be expressed as the number of tuples of a given size whose product lies below a prescribed threshold. This allows us to work directly with the set’s cardinality, thereby avoiding integration over a high-dimensional region and considerably simplifying the computation.

To see this more clearly, note that by Proposition~\ref{prop:eigenvalues_operator}, we have that the cardinality of the set $\{\lambda_{m}: \lambda_m \geq \varepsilon \}$ is the same as for $\{\beta: \sigma_1^{\beta_1} \cdots \sigma_d^{\beta_d} \geq \varepsilon\}$. Then, since $\sigma_{j} = C_{\alpha} j^{-\alpha}$, this give us: 
\begin{equation}
    \left |\{\lambda_{m}: \lambda_m \geq \varepsilon \} \right | = \left | \left\{\beta: \prod_{a=1}^{d} a^{\beta_a} \leq L \right\} \right |,
\end{equation}
for $L = \left (\sfrac{C_{\alpha}}{\varepsilon} \right )^{\alpha}$. The cardinality of the right-hand side has been well-studied for integer numbers and corresponds to a classical problem in number theory (c.f. \cite{T15}, Chapter I.3). In particular, this cardinality is given by: 
\begin{equation}
    \left | \left\{\beta: \prod_{a=1}^{d} a^{\beta_a} \leq L\right\} \right | = CL\mathrm{poly}\log(L),
\end{equation}
for a constant $C$ independent on the dimension $d$. Note that this maps $M(\varepsilon) = |\{\lambda_{m}: \lambda_{m} \geq \varepsilon \}|$ to an integer. We can then invert this relation to get an eigenvalue $\varepsilon$ as a function of $M$ and conclude the desired result.
\end{proof}

An interesting consequence of the above result is that for $\alpha>1$, since $r_0(\Sigma) = O(1)$ (c.f. \cref{eq:asymptotics_effective dimension_1}) the spectrum of this class of inner-product kernels satisfy a capacity condition with the same exponent of the data covariance $\lambda_m = \Theta(m^{-\alpha})$. This is illustrated in \cref{fig:experiment_power_law_decay}.  

We can further extend Corollary~\ref{cor:power_law_decay_monomial} to any finite-degree polynomial kernel, when $\alpha \in [0, \frac{1}{\ell})$, for some $\ell \in \NN$.

\begin{proposition}
    Let $\ell \in \NN$, $\alpha \in [\frac{1}{\ell + 2} ,\frac{1}{\ell + 1})$, and $D \gg L$. Let $\lambda_m$ denote the $m-th$ eigenvalue of the kernel $k(x,x') = \sum_{j=0}^D h_{j} \langle x,x'\rangle^k$, with $x,x' \sim \gamma_d^{\alpha}$. Denote $B_{d,j}:=\binom{d + j}{j}$. Then: 
    \begin{itemize}
    \item \textbf{Spectral Gap Sector:} If $B_{d,j}\leq m \leq B_{d,j+1}$, for $j \leq \ell$, denote by $m^{+}: m - B_{d,j}$. Then: 
    \[
   \lambda_m = \tilde{\Theta}\left ( C_1 \dfrac{\left (m^+\right )^{-\alpha}}{r_0(\Sigma)^{j+1}} \right ).
    \]
    \item \textbf{Continuous Spectrum:} If $m > B_{d,\ell}$, then there exists a strictly increasing sequence of numbers $a_\ell, \dots a_{D-1}$, such that $a_{j} = O(d^{j+1}\mathrm{poly}\log(d))$, such that if $a_j \leq m \leq a_{j+1}$, then there exists constants $C_3, C_4$, independent of the dimension, such that: 
    \[
   \lambda_m  = \tilde{\Theta} \left ( C_4 \dfrac{\left ( m -a_{j}\right )^{-\alpha}}{r_0(\Sigma)^{j+1}} \right ),
    \]
    \end{itemize}
    where we used $\tilde{\Theta}$ to hide the poly-logarithmic factors. 
    \label{prop:spectrum_power_law_final}
\end{proposition}

The intuition behind Proposition~\ref{prop:spectrum_power_law_final} is the following: For a given value of $\alpha \in [0,1)$, we can say precisely how many spectral gaps are in the spectrum. This is illustrated in Figure~\ref{fig:tikz_power_law_decay}. We get two different behaviors: When there are spectral gaps (which correspond to the first part of the proposition), we will have the same behavior described by Corollary~\ref{cor:power_law_decay_monomial} (see LHS of Figure~\ref{fig:tikz_power_law_decay}). When $\alpha > 0$, after a finite number of spectral gaps there is a part of the spectrum that is continuous. This part is described by the second part of the Proposition. For the details of the proof, we refer the reader to Appendix~\ref{Appendix:Gramm_Schmidt}. 

\begin{figure}[t]
        \centering
        \includegraphics[width=0.45\linewidth]{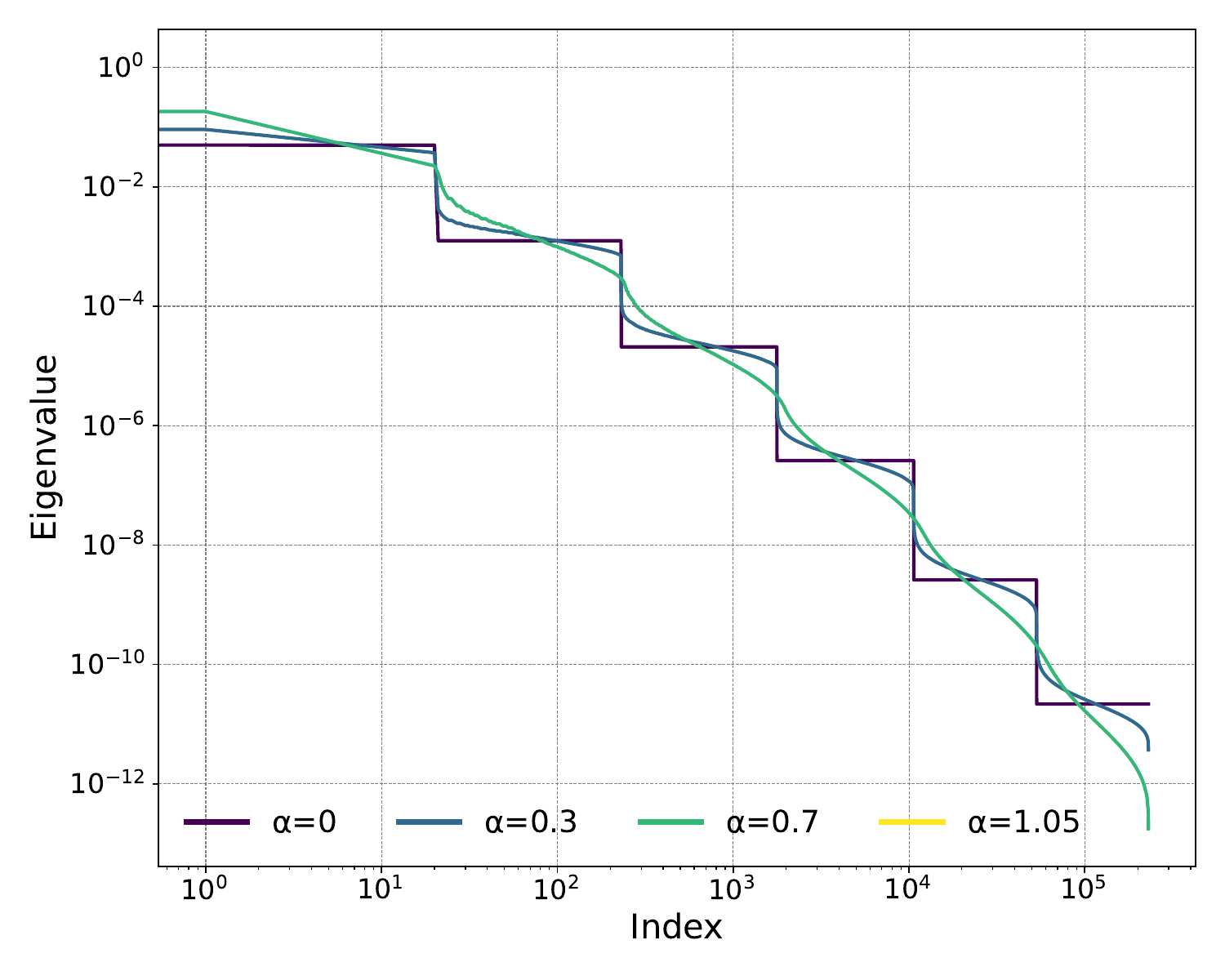}    
        \includegraphics[width=0.45\linewidth]{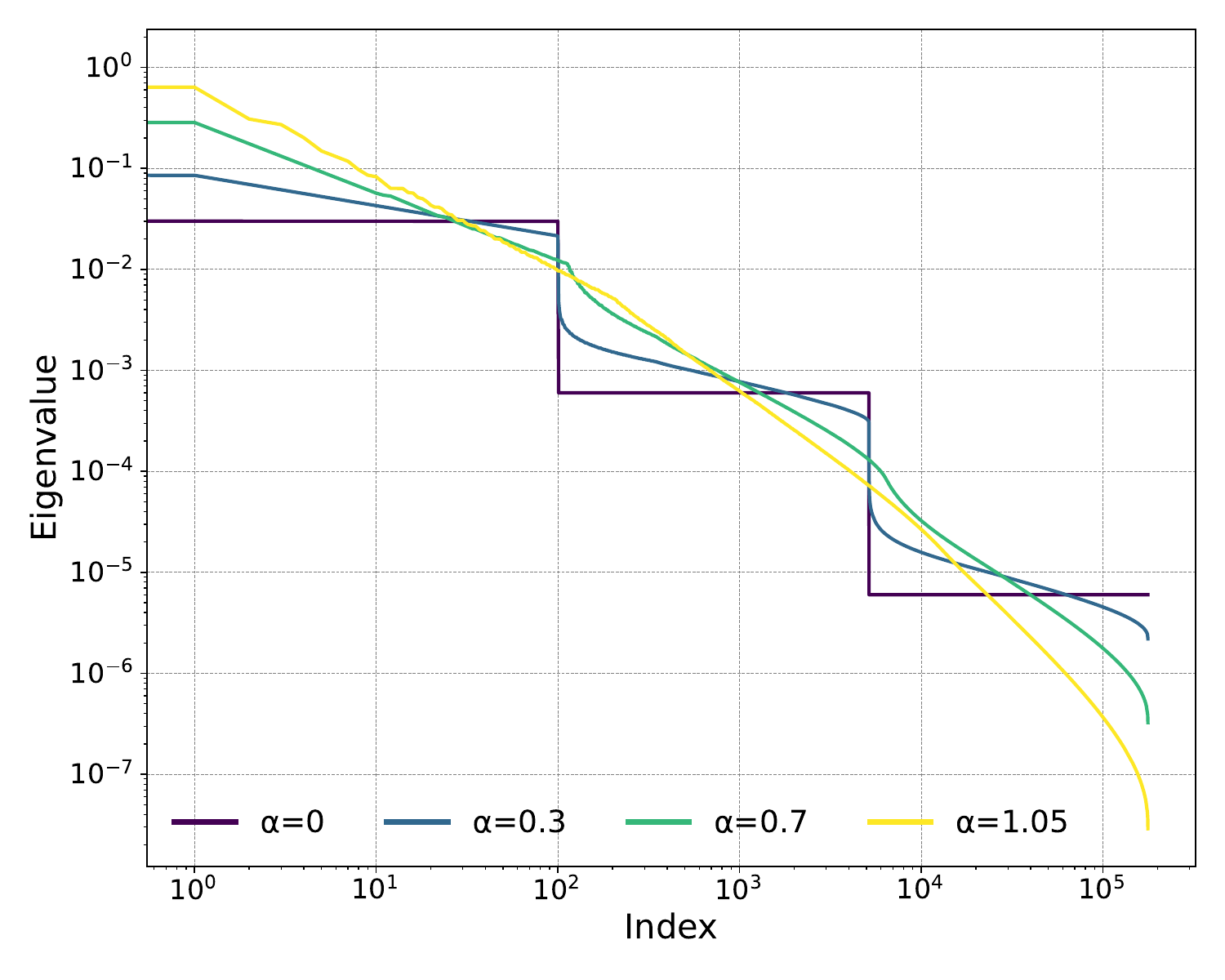}
    \caption{\textbf{Left}: Theoretical Spectrum the kernel resulting by truncating $k(x,x') = \exp(\langle x, x' \rangle) $ on the 5-th degree of it's Taylor expansion, and with $x,x' \sim \gamma_d^{\alpha}$ for $\alpha \in \{ 0, 0.3, 0.7, 1.05\}$, with $d = 20$. \textbf{Right}: Theoretical spectrum of the kernel $k(x,x') = ( 1 + \langle x, x' \rangle)^3 $, with $x,x' \sim \gamma_d^{\alpha}$ for $\alpha \in \{ 0, 0.3, 0.7, 1.05\}$, with $d = 100$.}
    \label{fig:experiment_1}
\end{figure}

\subsection{Consequences for learning}
\label{sec:learning}
We now turn to our second main result, which addresses how anisotropy in the data affects the generalisation capacity of kernel ridge regression. Intuitively, since the effective dimension satisfies $r_{0}(\Sigma)\lesssim d$ (cf.\ \cref{eq:asymptotics_effective dimension_1}) and decreases with $\alpha\geq 0$, one expects that strongly anisotropic data should reduce the sample complexity required to achieve small excess risk. The result in this section confirm this intuition and provides a precise characterization of the benefits of anisotropy in the high-dimensional regime $\alpha \in [0,1)$.

Our focus in this section will be on the following Hermite polynomial kernel:
\begin{equation}
    k(x,x') =  \sum_{\beta \in \ZZ^{d}_{\geq 0}} \xi_{\beta}\binom{|\beta|}{\beta} \sigma^{\beta} \He_{\beta}(\Sigma^{-\sfrac{1}{2}}x) \He_{\beta}(\Sigma^{-\sfrac{1}{2}}x'), 
    \label{def:hermite_kernel}
\end{equation}
with $\xi_{\beta} \geq 0$ for all $\beta \in \ZZ^{d}_{\geq 0}$, $\binom{|\beta|}{\beta}:=\dfrac{|\beta |!}{\beta_1! \cdots \beta_d!}$ and $\sigma^{\beta}:= \sigma_1^{\beta_1}\cdots \sigma_d^{\beta_d}$. 

Characterizing the excess risk requires a close control not only of the spectrum of the kernel but also of its eigenfunctions. Diagonalizing a general kernel in dimension $d$ is a challenging mathematical problem, with explicit solutions only known for particular cases, such as harmonic polynomials. For this reason, a common simplification in the theoretical literature consists of studying kernels which are directly defined in terms of their Mercer decomposition, see for instance  \citep{FB24,BBPV23}. It is an interesting open question to  find a Mercer Decomposition for inner-product kernels, as the ones considered in \cref{Assumptions_h}, with anisotropic Gaussian data. 

\begin{remark}[Gaussian kernel]
By Mehler's formula (c.f. \cite{B23_blog}), taking $\xi_{\beta} = \xi_{|\beta|}$ for all $\beta \in \ZZ^{d}_{\geq 0}$  and $\xi_{m}>0$ for all $m \geq 0$,  the Hermite kernel in \cref{def:hermite_kernel} corresponds to a Gaussian RBF Kernel $G(x, x') = \exp(-\frac{1}{2} (x-y)^{\top} T (x-y)$ with for a particular choice of p.s.d. matrix $T\in\mathbb{R}^{d\times d}$.
\end{remark}

Consider the KRR problem defined in \cref{eq:def:krr} on the RKHS spanned by the Hermite kernel in \cref{def:hermite_kernel}. The minimizer is explicitly given by $\hat{f}_{\lambda}(x)=k_{x}^{\top}(K+\lambda I_{n})^{-1}y$, where $K_{ij}=k(x_{i},x_{j})$ is the kernel matrix and $k_{x}=k(x,x_{i})$. 
The main result result in this section states that in the high-dimensional regime of anisotropy $\alpha\in[0,1)$, under limited sample complexity $n = O_{d}(d^{\kappa})$ for some $\kappa>0$, this predictor only captures low-frequency components of the target function. More precisely, fix a small constant $\delta_0 > 0$ and define the subsets of multi-indices:
\begin{align*}
    \mathsf{High}(n) : = \left \{ \beta \in \ZZ^{d}_{\geq 0}: \sigma_1^{\beta_1} \cdots \sigma_d^{\beta_d} \leq \dfrac{1}{d^{\kappa + \delta_0}} \right \},
\end{align*}   
and 
\begin{align*}
\mathsf{Low}(n) : = \left \{\beta \in \ZZ^{d}_{\geq 0}: \sigma_1^{\beta_1} \cdots \sigma_d^{\beta_d} > \dfrac{1}{d^{\kappa + \delta_0}} \right \}.
\end{align*}   
This induces a decomposition of the kernel spectrum $\lambda_{\beta}$ into high- and low-frequency sectors, corresponding to $\mathsf{High}(n)$ and $\mathsf{Low}(n)$, respectively. At a high level, in the anisotropic high-dimensional regime $\alpha \in [0,1)$ with limited data $n=\Theta(d^{\kappa})$, the KRR predictor $\hat{f}_{\lambda}$ has sufficient resolution to capture only the low-frequency components of the target function. The high-frequency components remain unlearned, effectively behaving as an implicit ridge regularizer. This intuition is formalized in the following result.
\begin{theorem}
    \label{thm:generalization_error}
    Let $n = C d^{\kappa}$, with $\kappa > 0$, and $\alpha\in[0,1)$. Define $D(\kappa) = \lfloor \frac{\kappa}{1-\alpha} \rfloor$, and assume $D(\kappa)\cdot (1-\alpha) < \kappa$, and $\kappa \not = \lfloor \kappa \rfloor$. Let $\hat{f}_\lambda$ denote the KRR predictor in \cref{eq:def:krr} with Hermite kernel defined in \cref{def:hermite_kernel}, $\lambda >0$ denote the Ridge regularization and $ f_{\star}^{\mathsf{Low}(n)} \in \RR^{|\mathsf{Low}(n)|}$ denote the vector with all the Hermite coefficients of $f^\star$ for $\beta \in \mathsf{Low}(n)$. Then, 
    \[
    R(\hat{f}) = \| (I - S^{\mathsf{Low}(n)}) f_{\star}^{\mathsf{Low}(n)} \|_{L^2} + o_d(1),
    \]
    where $S^{\mathsf{Low}(n)} \in \RR^{|\mathsf{Low}(n)| \times |\mathsf{Low}(n)|}$ is the shrinkage matrix 
        \[
        S^{\mathsf{Low}(n)} = \left ( (\lambda  + \sigma_\mathrm{eff}) (nD) ^{-1} + I_n\right )^{-1},
        \]
        with $D \in \RR^{|\mathsf{Low}(n)| \times |\mathsf{Low}(n)| }$ a diagonal matrix indexed by $\beta \in \mathsf{Low}(n)$, and with elements
        \[
        D_{\beta, \beta} = h_{|\beta|} |\beta|! \dfrac{ \sigma_1^{\beta_1} \cdots \sigma_d^{\beta_d}}{r_0(\Sigma)^{|\beta|}}.
        \]
        and $\sigma_\mathrm{eff}:= \gamma^{\mathrm{eff}} = \lambda + \sum_{\beta \in \mathsf{High}(n)} \lambda_{\beta}$ , with $\lambda_\beta$ the eigenvalues of the kernel. 
\end{theorem}

For a proof of this Theorem, we refer the reader to Appendix~\ref{appendix:Kernel_Matrix_Approximation}.
\begin{remark}
A few remarks about \cref{thm:generalization_error} are in order.
\begin{itemize}
    \item \Cref{thm:generalization_error} relies on a concentration argument for the kernel matrix in the high-dimensional regime. Consequently, it applies only to $\alpha \in [0,1)$, where the effective dimension $r_{0}(\Sigma)$ diverges with $d$. For $\alpha \geq 1$, the data becomes effectively low-dimensional, and obtaining a comparably fine characterization of the excess risk requires random matrix theory techniques (see, e.g.\ \citep{defilippis2024dimension}).
    \item We exclude the case where $\kappa$ is an integer. This restriction arises because results of this type require concentration of a covariance matrix with $|\mathsf{Low}(n)|$ features, which in turn requires $n \gg |\mathsf{Low}(n)|$. This concentration becomes particularly challenging when $\alpha>\sfrac{1}{2}$, owing to the absence of a spectral gap (see \cref{prop:eigenvalues_operator}). 
    \item The proof of \cref{thm:generalization_error} builds on Theorem~4 of \citep{MMM22}, adapted to our setting, and requires establishing a number of non-trivial conditions on the kernel operator. A key step is the concentration of the diagonal entries of the kernel matrix, which we establish for our anisotropic kernel.
    \item By taking $\alpha=0$, \cref{thm:generalization_error} yields a result similar to \cite{GMMM20} for inner-product kernels with isotropic data on the sphere, therefore also generalizing their result to i.i.d.\ Gaussian setting. 
\end{itemize}
\end{remark}
Note that eigenvalues in $\mathsf{High}(n)$ correspond to Hermite polynomials of degree $D(\kappa)$ or higher. However, $\mathsf{Low}(n)$ also contains certain polynomials of degree exactly $D(\kappa)$. This observation yields the following corollary of \cref{thm:generalization_error}.
\begin{corollary} 
    Under the same assumptions of \cref{thm:generalization_error} the KRR predictor $\hat{f}_\lambda$ is at most a polynomial of degree $D(\kappa)$. In particular, there exist polynomials of degree $D(\kappa)$ that can be learned in this regime. 
    \label{cor:degree_of_predictor}
\end{corollary}
\Cref{cor:degree_of_predictor} shows that the isotropic case ($\alpha=0$) is the worst case in the power-law data setting. Specifically, when $\alpha=0$ the predictor learns exactly a polynomial of degree $\lfloor \kappa \rfloor$, whereas for $\alpha>0$ it can only improve upon this, making $\lfloor \kappa \rfloor$ a lower bound on the degree of the learned polynomial. The dependence on the target function $f_{\star}$ comes from \cref{thm:generalization_error}: Anisotropy improves learning only when the target is well aligned with the eigenvectors of the data covariance; that is, when $f_{\star}(x)$ depends more strongly on the leading coordinates of $x$ (those with the largest variance) than on the trailing ones.

Altogether, this provides a clear answer to our initial question: overall, strong anisotropy on the data can only help the KRR predictor, being most beneficial when the target function has stronger alignment with the most important directions in data space. When this is not the case, for example when the target function is of the form $f_{\star}(x) = f_{\star} (x_{d})$, with $x_{d}$ the last coordinate of $x$ then \cref{thm:generalization_error} gives the same bound for all values of $\alpha \in [0,1)$. To illustrate this discussion, we consider two concrete examples. 

\begin{example}[Isotropic is the worst case] Consider the case when $f_{\star}(x) = \He_2 (x_{1})$, with $x_{1}$ the first  coordinate of $x$. Then, by \cref{thm:generalization_error} the sample complexity necessary to learn this function in the isotropic case $\alpha=0$ is $n=O(d^{2+\varepsilon})$, while for $\alpha > 0$, the sample complexity is $n=O(r_0(\Sigma)^{2+\varepsilon}) = O(d^{2(1-\alpha) + \varepsilon}) \ll d^{2}$. Hence, learning this type of functions is easier for larger $\alpha$. 
\end{example}

\begin{example}[Alignment of the target] 
Consider the target function $f_{\star}(x)=\He_{2}(x_{d})$. In the isotropic case $\alpha=0$, the sample complexity is $n=O(d^{2+\varepsilon})$. For this target, anisotropy brings no advantage: \Cref{thm:generalization_error} shows that the required sample complexity is $n=O(\sigma_{d}^{-2-\varepsilon})=O(d^{2+\varepsilon})$ for any $\alpha\in(0,1)$, which coincides exactly with the isotropic rate.
\end{example}

\section{Numerical experiments}
In this section, we numerically illustrate the theoretical results of \cref{sec:mainres} through concrete examples, both within and beyond the scope of the mathematical assumptions, thereby showing the broader relevance of our findings. 

\subsection{Spectrum of inner product kernels}
We begin with an illustration of the theoretical eigenvalue predictions of Proposition~\ref{prop:eigenvalues_operator} and Corollary~\ref{corollary:continous_kernel_eigenvalues} for different kernels.

\Cref{fig:experiment_1} shows the spectrum of two kernels for different levels of anisotropy $\alpha$. In the isotropic case, the spectrum is piece-wise constant, with each level $m\geq 0$ corresponding to $\Theta(d^{m})$ degenerate eigenvalues of size $\Theta(d^{-m})$, a consequence of rotational symmetry \citep{GMMM20}. For $\alpha\in[0,1)$, this symmetry is broken, lifting the degeneracy of the eigenvalues. Nevertheless, \cref{prop:spectrum_power_law_final} shows that for the first few levels, a spectral gap remain, coinciding exactly with the isotropic levels. These spectral gaps have important consequences for learning, and is intimately connected to the existence of low- and high-frequency sectors in \cref{thm:generalization_error}. Both the size of the gaps as well as the size of the spectral gap region decrease with $\alpha\in[0,1)$, completely disappearing for $\alpha\geq 1$, for which the spectrum becomes purely continuous. 

Finally, we illustrate \cref{cor:power_law_decay_monomial} for $\alpha >1$ of a pure polynomial in \cref{fig:experiment_power_law_decay}, showing that this kernel satisfy a capacity condition with exponent equals to the data anisotropy.
\begin{figure}[t]
    \centering
    \includegraphics[width=0.75\linewidth]{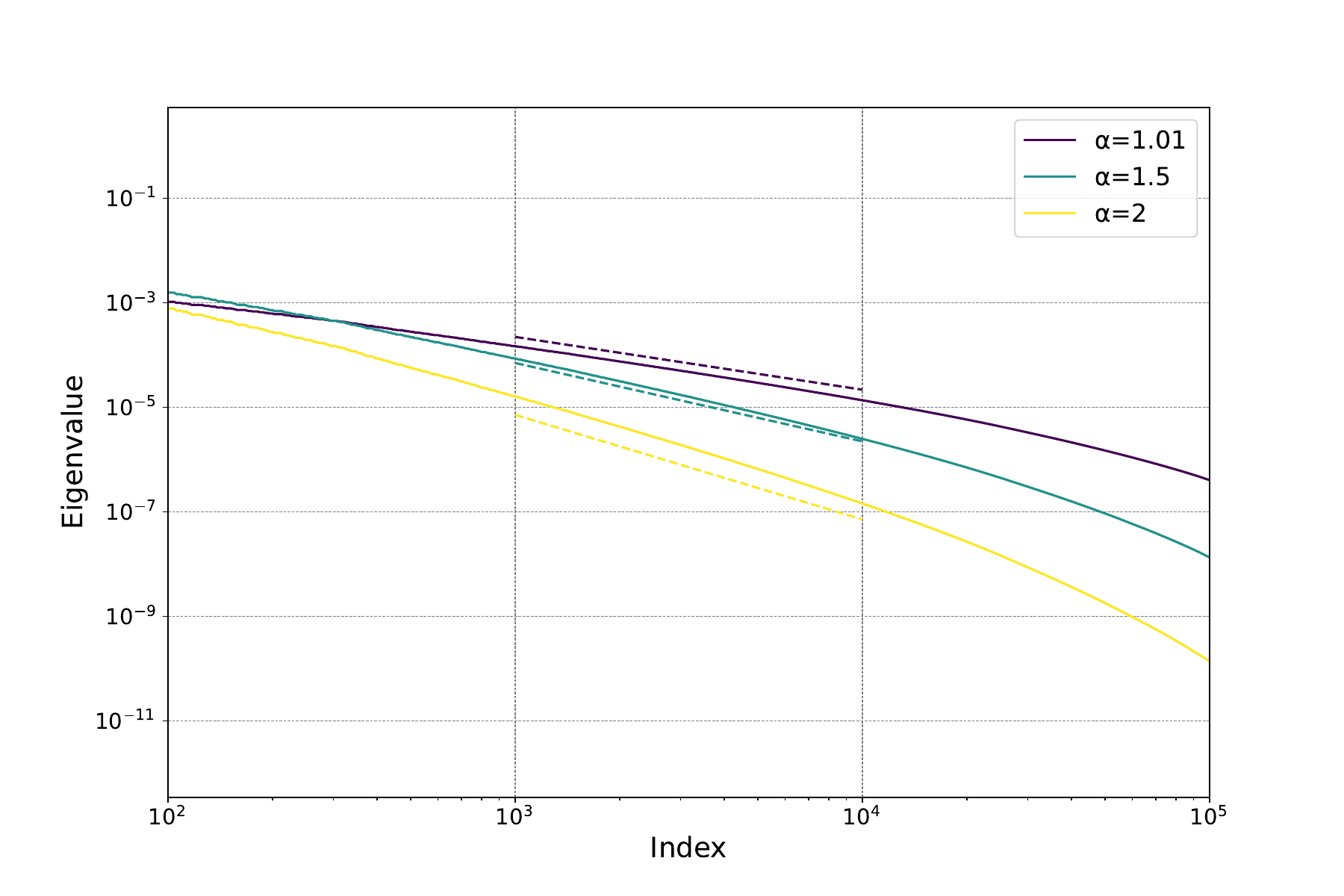}
    \caption{The plot corresponds to the theoretical spectrum of a polynomial kernel $K(x,x') = \langle x,x'\rangle^3$ with $d = 100$. Dashed lines correspond to function $C \cdot i^{-\alpha}$ for each value of $\alpha \in \{1.01, 1.5, 2\}$.  }
    \label{fig:experiment_power_law_decay}
\end{figure}

\subsection{Excess Risk for different targets} 
We now illustrate the generalization results in \cref{thm:generalization_error} and \cref{cor:degree_of_predictor}. \Cref{fig:experiment_2_1} shows the excess risk in \cref{eq:risk} for the Hermite kernel defined in \cref{def:hermite_kernel} and different training data sizes. Each sub-figure correspond to a different choice of target function $f_{\star}$.\\
 
The left side of \Cref{fig:experiment_2_1} corresponds to a target function which depends only the the first coordinate of the covariates:  $f_{\star}(x) = \He_{1}(z_{1}) + \He_{2}(z_{1})  + \He_{3}(z_{1})$, where $z_{1} = (\Sigma^{-\sfrac{1}{2}}x)_{1}$. As discussed in \cref{sec:learning}, this corresponds to a case in which anisotropy strongly helps generalization. Indeed, in the isotropic case $\alpha=0$ (purple curve), the error quickly plateau at this range of $n$, while in for high-anisotropy $\alpha=0.9$ (yellow curve) approaches zero at the same range --- a consequence of the fact that learning polynomials of the first coordinate require polynomial sample complexity in the effective dimension $r_{0}(\Sigma)$, which for $\alpha>0$ can be much smaller than $d$.\\ 

The right side of \Cref{fig:experiment_2_1} corresponds to the extreme opposite case: a target that depends only on the last coordinate of the covariates: $f_{\star}(x) = \He_{1}(z_{d}) + \He_{2}(z_{d})  + \He_{3}(z_{d})$. As discussed in \cref{sec:learning}, this corresponds to a case in which anisotropy does not generalization. Indeed, in this case the excess risk for the anisotropic kernels plateau at the same risk as the isotropic case. 
\begin{figure}[t]
        \centering
        \includegraphics[width=0.45\linewidth]{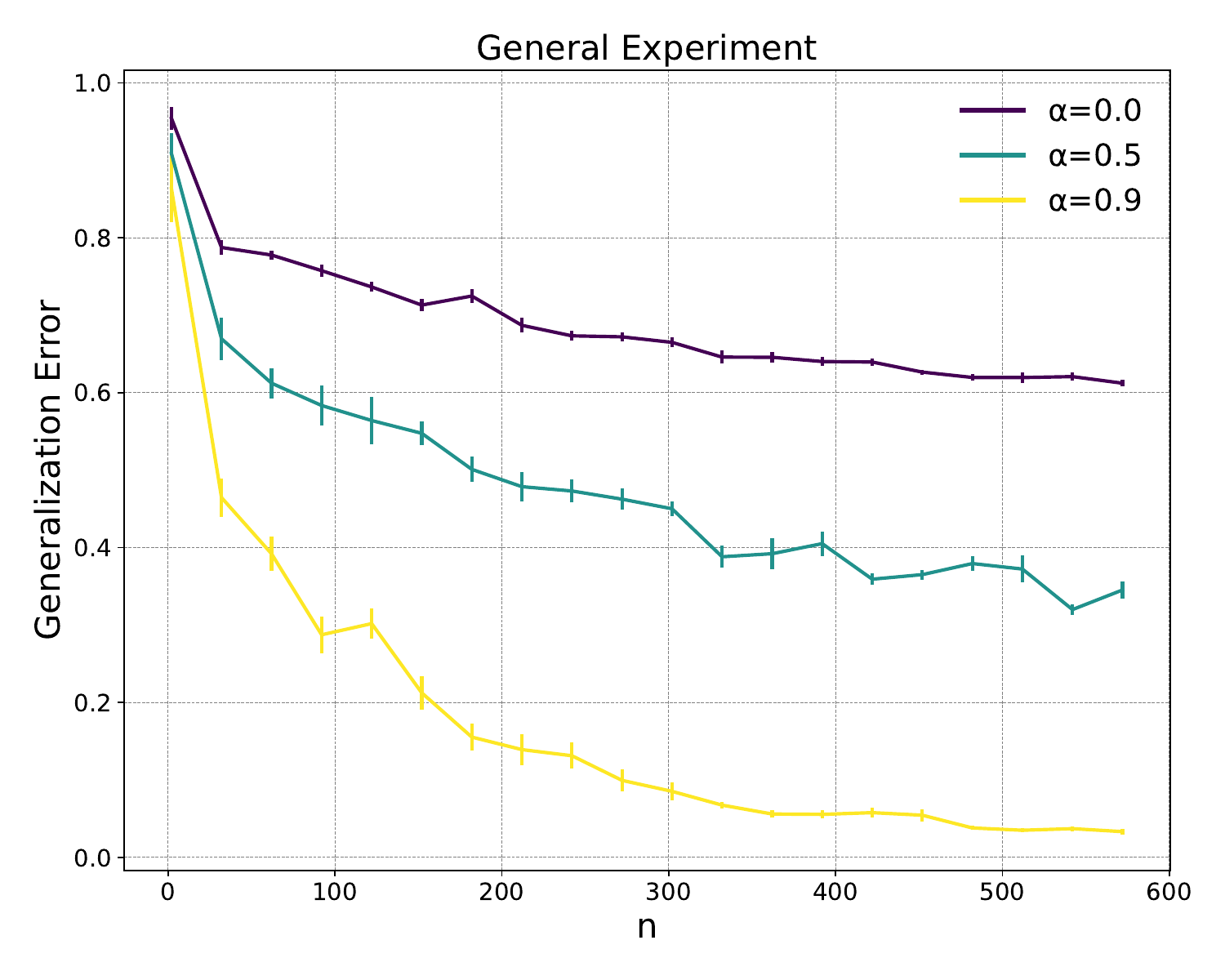}    
        \includegraphics[width=0.45\linewidth]{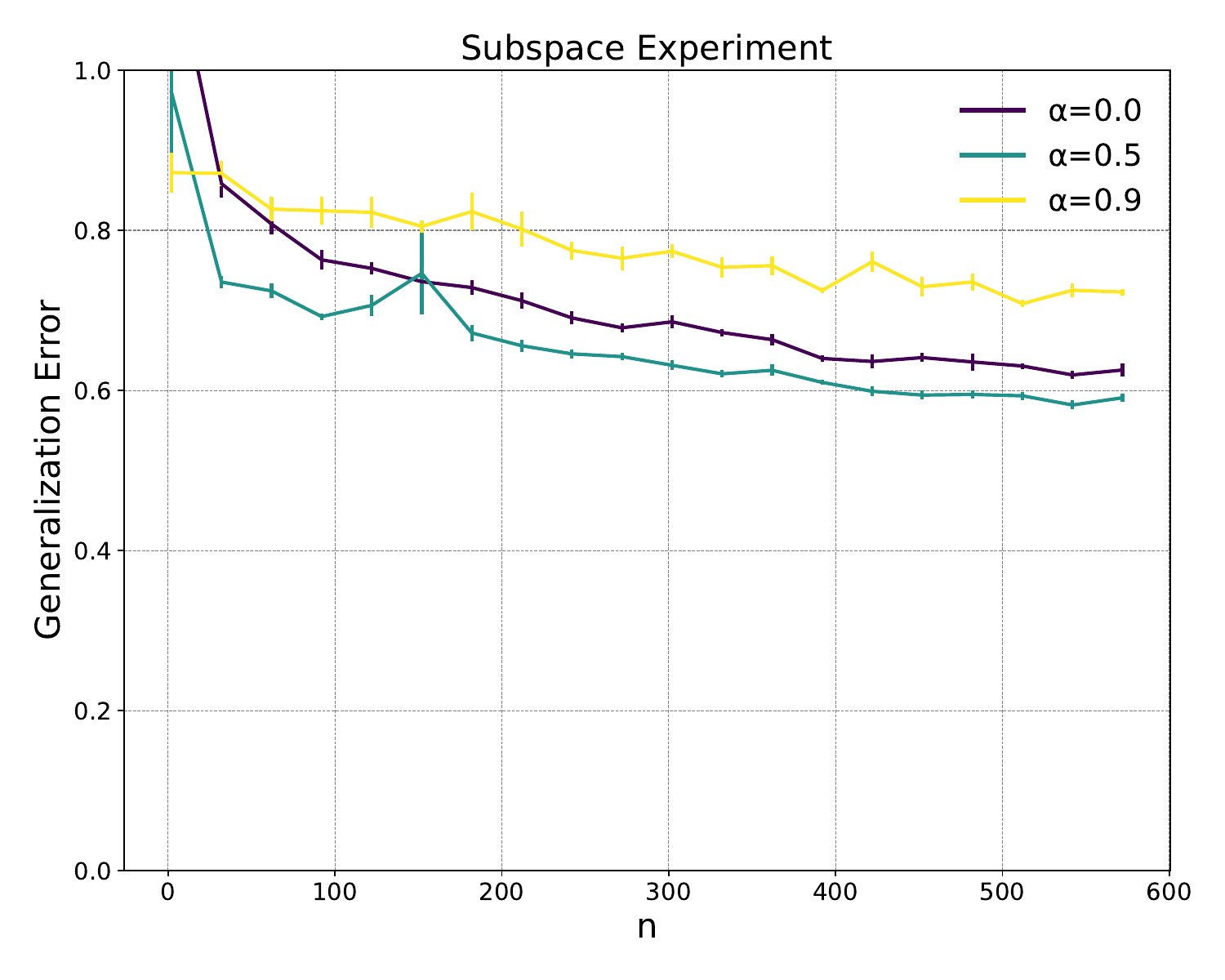}
    \caption{Excess risk for the kernel in Equation~\eqref{def:hermite_kernel} maximum degree equal to $3$ with $d = 100$, $\lambda =0.01$. The target function is of the form
    $f_{\star}(x) = \He_{1}(z_{i}) + \He_{2}(z_{i})  + \He_{3}(z_{i})$. In the first plot (Left) , we take $i=1$, while in the second (Right), we take $i=d$. Plots are obtained by averaging $10$ seeds, and bars denote the standard deviation. }
    \label{fig:experiment_2_1}
\end{figure}

\section*{Conclusion}
In this work we studied the spectral and generalization properties of KRR under anisotropic power-law data. Our results bridge two previously disconnected approaches to generalization in KRR: the high-dimensional analysis of isotropic data and the classical source–capacity framework. A key takeaway is that power-law anisotropy is benign for generalization, with the largest benefits arising when the target function aligns with the highest-variance components of the data. In this case, for a fixed sample complexity, anisotropy enables the predictor to capture higher-frequency components of the target than in the isotropic setting, a phenomenon our results characterize precisely. Looking ahead, several interesting directions remain open, including extending our excess risk analysis to more general inner-product kernels and developing a characterization of the excess risk in the regime $\alpha \geq 1$. 
\section*{Acknowledgements}
We would like to thank Francis Bach, Gérard Ben Arous, Yatin Dandi, Florentin Guth, Florent Krzakala, Fanghui Liu, Theodor Misiakiewicz and Eliot Paquette for insightful discussions. This work was supported by the French government, managed by the National Research Agency (ANR), under the France 2030 program with the reference ``ANR-23-IACL-0008'' and the Choose France - CNRS AI Rising Talents program. 

\bibliography{References}
\clearpage
\clearpage
\bibliographystyle{abbrvnat}

%
%

\onecolumn
\newpage
\appendix
\section{Change of Basis Matrix}
\label{Appendix:Gramm_Schmidt}

In this section, we will work with a $D$-degree kernel $k$ of the following form:  
\[
k(x,x') = \sum_{k =0}^D h_{k} \langle x, x' \rangle^{k}, a_{k} \geq 0 \forall k \in [D]. 
\]
with $x,x' \in \RR^{d}$ with distribution $x,x' \mathcal{N}(0,\Sigma)$, and $\Sigma = \mathrm{diag}(\sigma_1, \dots, \sigma_d)$ and $a_k  \geq 0$. The ideas from this Appendix are closely related to \citep{LRZ20}, with the difference that since we are directly working with Gaussians, we can arrive to explicit expressions. If we further expand the inner product and write $z_i = \Sigma^{-\frac{1}{2}}x_i$, we get: 
\[
k(x,x') = \sum_{k =0}^D a_k \sum_{|\beta|=k} \binom{k}{\beta_1, \dots, \beta_d} x^{\beta} x'^{\beta} = \sum_{k =0}^D a_k \sum_{|\beta|=k} \binom{k}{\beta_1, \dots, \beta_d} \sigma_1^{\beta_1} \cdots \sigma_d^{\beta_d} z^{\beta} z'^{\beta}.
\]
Now, consider $n$ independent samples $x_1, \dots x_n$, and the kernel matrix associated to $k$, which we denote $k \in \RR{n \times n}$. For each $i \in [n]$ and each multi-index $\beta \in \ZZ^{d}_{\geq 0}$ with $|\beta|\leq D$,  let $\Phi_{i,\beta} \in \RR$ be defined by 
\begin{equation}
    \Phi_{i, \beta} = \sqrt{h_{k} \binom{|\beta|}{\beta_1, \dots, \beta_d} \sigma_1^{\beta_1} \cdots \sigma_d^{\beta_d}} z_i^{\beta},
\end{equation}
and let $\Phi_i \in \RR^{\binom{d+D}{D}}$ be defined by $\Phi_i = (\Phi_{i,\beta})_{|\beta| \leq D}$.  Then, we have that: 
\begin{equation}
\PP_{i,j} = \Phi_{i}^T \Phi_{j}  \quad \forall i, j \in [n]. 
\end{equation}
Now, for each $i \in [n]$, consider the vector $\Psi_i \in \RR^{\binom{d+D}{D}}$ with coordinates
\begin{equation}
    \Psi_{i, \beta} =  \sqrt{ \underbrace{h_{|\beta|} \binom{|\beta|}{\beta_1, \dots, \beta_d} \sigma_1^{\beta_1} \cdots \sigma_d^{\beta_d}}_{C_{\beta}}} He_\beta(z_i), \quad \beta \in \ZZ^{d}_{\geq 0}, |\beta |\leq D,
\end{equation}
where $He_\beta(z_i) = \prod_{a=1}^d he_{\beta_a}(z_a)$.  We will explicitly write a linear transformation $\Lambda \in \RR^{\binom{d+D}{D} \times \binom{d+D}{D}}$ so that $\Phi_i = \Lambda \Psi_i$. For this, we will use the following result:

\begin{lemma}
    Let $\beta \in ZZ^{d}_{\geq 0}$. Then, 
    \[
    z^{\beta} = \sum_{\substack{\bar{k} \leq \beta:  \\ \bar{k} = \beta \mod 2}} \left(\prod_{i=1}^d \frac{\beta_i!}{2^{(\beta_i-\bar{k}_i)/2}\,\big((\beta_i-\bar{k}_i)/2\big)! \sqrt{\bar{k}_i!}} \right) He_{\bar{k}}(z),
    \]
    where $He_{\bar{k}}$ denote the normalized Hermite polynomial in $\RR^{d}$, that is $He_{\bar{k}}(z) = he_{\bar{k}_1}(z_1) \cdots he_{\bar{k}_d}(z_d)$.
    \label{lemma:monomials_into_hermite_decomposition}
\end{lemma}
We omit the proof of Lemma~\ref{lemma:monomials_into_hermite_decomposition} as it is a direct computation involving derivatives of monomials.  Then, by Lemma~\ref{lemma:monomials_into_hermite_decomposition} we can re-write $ \Psi_{i, \beta}$ decomposing it into the Hermite basis: 
\begin{align}
    \Phi_{i, \beta} &= \sqrt{ h_{|\beta|}\binom{|\beta|}{\beta_1, \dots, \beta_d} \sigma_1^{\beta_1} \cdots \sigma_d^{\beta_d}} z_i^{\beta} \\ 
    & =  \sqrt{  h_{|\beta|}\binom{|\beta|}{\beta_1, \dots, \beta_d} \sigma_1^{\beta_1} \cdots \sigma_d^{\beta_d}} \sum_{\substack{\bar{k} \leq \beta:  \\ \bar{k} = \beta \mod 2}} \left(\prod_{i=1}^d \frac{\beta_i!}{2^{(\beta_i-\bar{k}_i)/2}\,\big((\beta_i-\bar{k}_i)/2\big)! \sqrt{\bar{k}_i!}} \right) He_{\bar{k}}(z_i) \\ 
    & = \sqrt{h_{|\beta|} \binom{|\beta|}{\beta_1, \dots, \beta_d} \sigma_1^{\beta_1} \cdots \sigma_d^{\beta_d}} \sum_{\substack{\bar{k} \leq \beta:  \\ \bar{k} = \beta \mod 2}} \left(\prod_{i=1}^d
    \frac{\beta_i!}{2^{(\beta_i-\bar{k}_i)/2}\,\big((\beta_i-\bar{k}_i)/2\big)!\sqrt{\bar{k}_i!}}
    \right) \dfrac{\Psi_{i,\bar{k}}}{\sqrt{C_{\bar{k}}}} \\ 
    & = \sum_{\substack{\bar{k} \leq \beta:  \\ \bar{k} = \beta \mod 2}} \sqrt{ h_{|\beta|}\binom{|\beta|}{\beta_1, \dots, \beta_d} \sigma_1^{\beta_1} \cdots \sigma_d^{\beta_d}} \left(\prod_{i=1}^d \frac{\beta_i!}{2^{(\beta_i-\bar{k}_i)/2}\,\big((\beta_i-\bar{k}_i)/2\big)!} \right) \dfrac{\Psi_{i,\bar{k}}}{\sqrt{\bar{k}_i! C_{\bar{k}}}} 
\end{align}
Thus, we can define 
\begin{equation}
    \Lambda_{\beta, \bar{k}} = \left [\substack{\bar{k} \leq \beta:  \\ \bar{k} = \beta \mod 2} \right ]  \sqrt{ h_{|\beta|}\binom{|\beta|}{\beta_1, \dots, \beta_d} \sigma_1^{\beta_1} \cdots \sigma_d^{\beta_d}} \left(\prod_{i=1}^d \frac{\beta_i!}{2^{(\beta_i-\bar{k}_i)/2}\,\big((\beta_i-\bar{k}_i)/2\big)!} \right)  \dfrac{1}{\sqrt{\bar{k}_i! C_{\bar{k}}}}
\end{equation}
We can further manipulate this expression by inserting the definition of $C_{\bar{k}}$: 
\begin{align}
     \Lambda_{\beta, \bar{k}} &= \left [\substack{\bar{k} \leq \beta:  \\ \bar{k} = \beta \mod 2} \right ]  \sqrt{ h_{|\beta|}\binom{|\beta|}{\beta_1, \dots, \beta_d} \sigma_1^{\beta_1} \cdots \sigma_d^{\beta_d}} \left(\prod_{i=1}^d \frac{\beta_i!}{2^{(\beta_i-\bar{k}_i)/2}\,\big((\beta_i-\bar{k}_i)/2\big)!} \right)  \dfrac{1}{\sqrt{\bar{k}_i! C_{\bar{k}}}} \\ 
     & = \left [\substack{\bar{k} \leq \beta:  \\ \bar{k} = \beta \mod 2} \right ]  \sqrt{  h_{|\beta|}\binom{|\beta|}{\beta_1, \dots, \beta_d} \sigma_1^{\beta_1} \cdots \sigma_d^{\beta_d}} \left(\prod_{i=1}^d \frac{\beta_i!}{2^{(\beta_i-\bar{k}_i)/2}\,\big((\beta_i-\bar{k}_i)/2\big)!} \right)  \dfrac{1}{\sqrt{\bar{k}_i! }} \sqrt{\dfrac{1}{h_{|\bar{k}|} \binom{|\bar{k}|}{\bar{k}_1, \dots, \bar{k}_d} \sigma_1^{\bar{k}_1} \cdots \sigma_d^{\bar{k}_d}} } \\ 
     & = \left [\substack{\bar{k} \leq \beta:  \\ \bar{k} = \beta \mod 2} \right ]  \sqrt{  \dfrac{h_{|\beta|}\binom{|\beta|}{\beta_1, \dots, \beta_d}}{h_{|\bar{k}|}\binom{|\bar{k}|}{\bar{k}_1, \dots, \bar{k}_d}} \sigma_1^{\beta_1-\bar{k}_1} \cdots \sigma_d^{\beta_d-\bar{k}_d}} \left(\prod_{i=1}^d \frac{\beta_i!}{2^{(\beta_i-\bar{k}_i)/2}\,\big((\beta_i-\bar{k}_i)/2\big)!} \right)  \dfrac{1}{\sqrt{\bar{k}_i! }} \label{eq:def_Lambda_beta_bar_k} \\
     & = O \left ( \left [\substack{\bar{k} \leq \beta:  \\ \bar{k} = \beta \mod 2} \right ]  \sqrt{\dfrac{\sigma_1^{\beta_1-\bar{k}_1} \cdots \sigma_d^{\beta_d-\bar{k}_d}}{\deff^{|\beta|-|\bar{k}|}}} \right ) 
\end{align}
Note that 
\[
\Lambda_{\beta,\beta } = \sqrt{\beta_1 \cdots \beta_d!},
\]
and $\Lambda$ is a upper-triangular matrix, so $ \max \{\| \Lambda  \|_{op}, \| \Lambda^{-1}\|_{op}\} \leq C(D)$.  As we will see now, this construction will be fundamental in characterizing the spectrum of $P$ as an operator in $L^{2}(\gamma_d^{\alpha})$. Note that this is not the same as the empirical spectrum of the kernel matrix  $K$. \\

With this definition of $\Lambda$, we can write $\Phi_{i}$ as a linear transformation of $\Psi_{i}$. First: 
\begin{equation}
    \Phi_{i,\beta} = \sum_{\substack{\bar{k} \leq \beta \\ \bar{k} \equiv_2 \beta}} \Lambda_{\beta, \bar{k}} \Psi_{i, \bar{k}} = \sum_{\bar{k} \in \ZZ^{d}_{\geq 0}: |\bar{k}|\leq D} \Lambda_{\beta, \bar{k}} \Psi_{i, \bar{k}},
\end{equation}
and then
\begin{equation}
    \Phi_{i} = \Lambda \Psi_{i}.
\end{equation}
In matrix form, for $\Phi = [\Phi_{1}^T, \dots, \Phi_{n}^T]^T \in \RR^{\binom{d + D}{D} \times \binom{d + D}{D} }$, $\Psi = [\Psi_{1}^T, \dots, \Psi_{n}^T]^T \in \RR^{\binom{d + D}{D} \times \binom{d + D}{D} }$:
\begin{equation}
    \Phi = \Psi \Lambda. 
\end{equation}

We summarize this in the following Lemma, which is analogous to Proposition 1 in \citep{LRZ20}. 

\begin{lemma} Consider $x \sim \mathcal{N}(0,\Sigma)$, with $\Sigma = \mathrm{diag}(\sigma_1, \dots, \sigma_d)$, and let $P: \RR^d \times \RR^d \to \RR$ be the polynomial kernel 
\[
k(x,x') = \sum_{k =0}^D a_{k} \langle x, x' \rangle^{k},
\]
with $a_{k} \not =0 \forall k \in [D]$. Then, there exists an upper-triangular matrix $\Lambda \in \RR^{\binom{d + D}{D} \times \binom{d + D}{D}}$, which we index by multi-indices $\beta, \bar{k} \in \ZZ^{d}_{\geq 0}$ with $|\beta|, |\bar{k}| \leq D$, defined by
\[
\left [\substack{\bar{k} \leq \beta:  \\ \bar{k} = \beta \mod 2} \right ]  \sqrt{ \dfrac{h_{|\beta|}\binom{|\beta|}{\beta_1, \dots, \beta_d}}{h_{|\bar{k}|}\binom{|\bar{k}|}{\bar{k}_1, \dots, \bar{k}_d}} \sigma_1^{\beta_1-\bar{k}_1} \cdots \sigma_d^{\beta_d-\bar{k}_d}} \left(\prod_{i=1}^d \frac{\beta_i!}{2^{(\beta_i-\bar{k}_i)/2}\,\big((\beta_i-\bar{k}_i)/2\big)!} \right)  \dfrac{1}{\sqrt{\bar{k}_i! }},
\]
such that for two samples $x_i, x_j$, 
\[
P(x_i,x_j) = \Psi_{i}^T \Lambda^T  \Lambda \Psi_{j},
\]
with $\Psi \in \RR^{\binom{d+D}{D}}$ given by the Hermite polynomials features:
\[
\Psi_{i, \beta} =  \sqrt{h_{|\beta|} \binom{|\beta|}{\beta_1, \dots, \beta_d} \sigma_1^{\beta_1} \cdots \sigma_d^{\beta_d}} He_\beta(z_i), \quad \beta \in \ZZ^{d}_{\geq 0}, |\beta |\leq D.
\]
Moreover, $\max \{\| \Lambda\|_\mathrm{op}, \| \Lambda^{-1} \|_{op} \} \leq C(D)$. 
\label{lemma:Lambda_matrix_and_change_of_basis}
\end{lemma}

\subsection{Relation to the Eigenvalues of the Kernel Operator}

In this section, we will explain how we can use the matrix we constructed in the last section to obtain the eigenvalues of the truncated kernel operator. This argument is a modification of the one in \cite{LL13}. \\

We begin by writing the eigenvalue problem for the kernel 
\[
k(x,x') = \sum_{k =0}^D h_{k} \langle x, x' \rangle^{k}, h_{k} \geq 0 \quad \forall k \in [D]
\]
 as an operator in $L^2(\gamma^{\alpha}_d)$, where we denote $\gamma^{\alpha}_d$ as the gaussian measure we defined before with covariance parametrized by $\alpha$. Let $\varphi(x): \RR^d \to \RR^d$, and $\sigma \in \RR$. Then, our eigen problem is given by 
\begin{align}
    \lambda \varphi(x) & = \int_{\RR^d}\gamma_d^{\alpha}(dx') k(x,x') \varphi(x') \\ 
    & = \sum_{k = 0 }^D h_k \int_{\RR^d}\gamma^\alpha_d(dx') \langle x, x' \rangle^{k} \varphi(x') \\ 
    & = \sum_{k = 0 }^D  \sum_{\beta \in \ZZ^{d}_{\geq 0:}|\beta| = k} h_{|\beta|}\binom{k}{\beta_1, \dots, \beta_d}x^{\beta} \int_{\RR^{d}} \gamma^\alpha_d(dx') x'^{\beta}\varphi(x') \\ 
    & =  \sum_{k = 0 }^D h_k \sum_{\beta \in \ZZ^{d}_{\geq 0:}|\beta| = k}  \left (h_{|\beta|}\binom{|\beta|}{\beta_1, \dots, \beta_d} \right )^{\frac{1}{2}} x^{\beta} \underbrace{\int_{\RR^{d}}  \gamma^\alpha_d(dx')  \left ( h_{|\beta|}\binom{|\beta|}{\beta_1, \dots, \beta_d}\right )^{\frac{1}{2}} x'^{\beta}\varphi(x')}_{A_{\beta}}.
    \label{eq:def_A_beta}
\end{align}
Let 
\begin{equation}
    A_\beta:= \int_{\RR^{d}}\gamma^\alpha_d(dx') \left (h_{|\beta|}\binom{|\beta|}{\beta_1, \dots, \beta_d} \right )^{\frac{1}{2}}  x'^{\beta}\varphi(x')
    \label{eq:def_A_beta_2}
\end{equation}
Then, we can write equation \eqref{eq:def_A_beta} as: 
\begin{equation}
    \varphi(x) = \dfrac{1}{\lambda} \sum_{k = 0 }^D  h_k \sum_{\beta \in \ZZ^{d}_{\geq 0:}|\beta| = k} \left (h_{|\beta|}\binom{|\beta|}{\beta_1, \dots, \beta_d} \right )^{\frac{1}{2}}  A_{\beta}x^{\beta}.
\end{equation}
Replacing this the definition of $A_\beta$ in Equation \eqref{eq:def_A_beta_2} we get: 
\begin{align}
    A_\beta & = \int_{\RR^{d}} \gamma^\alpha_d(dx') \left (h_{|\beta|}\binom{|\beta|}{\beta_1, \dots, \beta_d}\right )^{\frac{1}{2}}  x'^{\beta} \left ( 
    \dfrac{1}{\lambda} \sum_{k = 0 }^D \sum_{\gamma \in \ZZ^{d}_{\geq 0:}|\gamma| = k} \left (h_{|\gamma|} \binom{|\gamma|}{\gamma_1, \dots, \gamma_d} \right )^{\frac{1}{2}}  A_{\gamma}x'^{\gamma}
    \right ) \\
    & = \dfrac{1}{\lambda}\sum_{k = 0 }^D  \sum_{\gamma \in \ZZ^{d}_{\geq 0:}|\gamma| = k}   \left ( h_{|\beta|}\binom{|\beta|}{\beta_1, \dots, \beta_d} \right )^{\frac{1}{2}}
    \left ( h_{|\gamma|}\binom{|\gamma|}{\gamma_1, \dots, \gamma_d} \right )^{\frac{1}{2}} A_\gamma \int_{\RR^{d}} \gamma^\alpha_d(dx') x^{\beta + \gamma}. 
\end{align}
Let $m_{\beta + \gamma}: = \int_{\RR^{d}} \gamma^\alpha_d(dx') x^{\beta + \gamma}$. Then: 
\begin{align}
 A_\beta & = \dfrac{1}{\lambda}\sum_{k = 0 }^D \sum_{\gamma \in \ZZ^{d}_{\geq 0:}|\gamma| = k}  \left ( h_{|\beta|}\binom{|\beta|}{\beta_1, \dots, \beta_d}\right )^{\frac{1}{2}}
    \left ( h_{|\gamma|} \binom{|\gamma|}{\gamma_1, \dots, \gamma_d}\right )^{\frac{1}{2}} A_\gamma  m_{\beta + \gamma}.
    \label{eq:A_beta_matrix_equation}
\end{align}
Let $\mathcal{S}^{D} = \left \{ \beta \in \ZZ^{d}_{\geq 0}: |\beta| \leq D \right \}$, and note that by a standard combinatorial argument, $|\mathcal{S}^D| = \binom{d+D}{D}$. Motivated by \eqref{eq:A_beta_matrix_equation}, we define the following matrix $M$ indexed by $\beta, \gamma \in \mathcal{S}^D$: 
\begin{equation}
    M_{\beta, \gamma}^{D} := \left (h_{|\beta|}\binom{|\beta|}{\beta_1, \dots, \beta_d}\right )^{\frac{1}{2}}
   \left ( h_{|\gamma|}\binom{|\gamma|}{\gamma_1, \dots, \gamma_d}\right )^{\frac{1}{2}}m_{\alpha + \gamma} . 
    \label{eq:def_M}
\end{equation}
Denote $A := (A_\beta)_{\beta \in \mathcal{S}^D}$. Then, we can re-write Equation \eqref{eq:A_beta_matrix_equation} by using this matrix obtaining:
\begin{equation}
     \lambda A = M A.
     \label{eq:eigenvalue_equivalence}
\end{equation}
Thus, we conclude that the eigenvalues  of the integral operator associated to the kernel $P$ are the same as the eigenvalues of the matrix $M^D$ from Equation \eqref{eq:def_M}. Thus, we can focus on studying the eigenvalues of $M^D$. Note that, by our construction in Proposition~\ref{lemma:Lambda_matrix_and_change_of_basis}, for any $i \in [n]$ 
\begin{equation}
    M = \EE[\Phi_i \Phi_i^T] = \EE[\Lambda \Psi_i \Psi_i^T \Lambda^T] = \Lambda \EE[\Psi_i \Psi_i^T]  \Lambda^T.
\end{equation}
Note that, by the orthogonality of Hermite polynomials, $\EE[\Psi_i \Psi^T]$ is a diagonal matrix with eigenvalues given by the expression in Proposition~\ref{prop:eigenvalues_operator}. On the other hand, by Ostrowski's Theorem ( \cite{HJ12}, Theorem 4.5.9), we get that since in our construction $\max \{\| \Lambda  \|_{op}, \| \Lambda^{-1}\|_{op}\} \leq C(D)$, we can conclude Proposition~\ref{prop:eigenvalues_operator}.  \\

\begin{remark}
    Note that, our procedure actually get's very precise eigenvalues: By writing the Singular Value Decomposition of $\Lambda$, we can actually see that the eigenvalues of $M$ will be exactly: 
    \[
    \lambda_{\beta} = h_{|\beta|}\binom{|\beta|}{\beta_1, \dots, \beta_d} \sigma_1^{\beta_1} \cdots  \sigma_1^{\beta_d} \cdot \beta_1! \cdots \beta_d ! = h_{|\beta|}|\beta|! \sigma_1^{\beta_1} \cdots  \sigma_1^{\beta_d}.
    \]
    \label{remark:exact_eigenvalues}
\end{remark}

\subsection{Proof of Corollary~\ref{corollary:continous_kernel_eigenvalues}}

Consider a function $h: \RR \to \RR$ satisfying Assumption~\ref{Assumptions_h}. Then, given $x,x' \sim \gamma_d^{\alpha}$, we have:
\begin{equation}
    k(x,x') = h(\langle x, x'\rangle) = \sum_{k \geq 0} h_{k } \langle x, x'\rangle^k. 
\end{equation}
We can re-write this as: 
\begin{equation}
    k(x,x') = k^{\leq D}(x,x') + k^{>D}(x,x'),
\end{equation}
for $k^{\leq D}(x,x') = \sum_{k=0}^{D}h_k \langle x,x'\rangle^{k}$, and $k^{>D}(x,x') =h^{>D}(\langle x, x'\rangle ) = \sum_{k > D} h_k \langle x, x'\rangle ^{k}$. We now recall the following useful inequality

\begin{lemma}[Hoffman-Wielandt Inequality, Theorem 2.2 in \cite{KG00}] 
If $A$ and $B$ are normal operators in $\RR^d$, in particular if they are symmetric, then
\[
\delta_2(\lambda(A), \lambda(B)) \leq \| A - B\|_{HS},
\]
where $\lambda(A), \lambda(B) \in \mathcal{\ell}^{2}(\RR)$ are the ordered eigenvalues of $A$ and $B$, and $\delta_2$ is given by
\[
\delta_2(\lambda(A), \lambda(B)) = \sum_{k \geq 0} (\lambda(A)_k - \lambda(B)_k)^2.
\]
\label{lemma:hoffman_wielandt}
\end{lemma}

From Lemma~\ref{lemma:hoffman_wielandt}, we get that: 
\begin{equation}
    \delta_2 (\lambda (k^{\leq D}), \lambda(k)) \leq \| k^{>D}\|_{HS}.
\end{equation}

By the smoothness assumptions we have on $h$, we can make the RHS as small as we want. In particular, if we fix a particular eigenvalue of $k^{\leq D}$, denoted by $\lambda_{\beta}$ for $\beta \in \ZZ^{d}_{\geq 0}$, then as long as $D$ is big enough so that $\| k^{>D}\|_{HS} \ll \lambda_{\beta}$, then we will have that there exists $\lambda(k)$, eigenvalue of $k$, and constants $c_1, c_2$ such that $ c_1 \lambda_\beta \leq \lambda(k) \leq c2\lambda(k)$. 

\newpage 
\section{Ordering the Spectrum}

For this section, most of the time we will write $A = C\cdot B$ to denote the fact that there exists constants $C_1, C_2$ such that $C_1 \cdot B \leq A \leq C_2 \cdot B $. We do this to avoid using cumbersome notation. \\ 

\subsection{Ordering the spectrum for Monomials}
We will first consider the particular case of the kernel $k:\RR^{d} \times \RR^{d} \to \RR$ given by $K(x,x') = \langle x, x'\rangle^{D}$ for some $D \in \NN$, and $x,x ' \sim \gamma_d^{\alpha}$ defined in \eqref{eq:cov_eigenvalues}. We can apply Proposition~\ref{prop:eigenvalues_operator} to get that, for all $\beta \in \ZZ^{d}_{\geq 0}$, there exists an eigenvalue $\lambda_{\beta}$, and constants (that don't depend on $\beta$) such that: 
\begin{equation}
    C_1 \sigma_1^{\beta_!} \cdots \sigma_d^{\beta_d} \leq \lambda_{\beta} \leq C_2 \sigma_1^{\beta_!} \cdots \sigma_d^{\beta_d}.
\end{equation}
Now, define $M(\varepsilon) := \left |\lambda: \lambda \geq \varepsilon \right | $. Then, we get: 
\begin{align}
    M(\varepsilon) & =  \left |\{ \lambda: \lambda \geq \varepsilon \}\right |  \\
    & =  \left |\{\beta \in \ZZ^{d}_{\geq 0}: |\beta|=D, \lambda_{\beta} \geq \varepsilon  \}\right | 
    \\
    & = \left |\{\beta \in \ZZ^{d}_{\geq 0}: |\beta|=D, \sigma_1^{\beta_1} \cdots \sigma_d^{\beta_d} \geq \varepsilon \} \right |.
\end{align}
Since by definition we have that $\sigma_j = C_{\alpha}j^{-\alpha} = \frac{j^{-\alpha}}{r_0(\Sigma)}$, we can re-write this as: 
\begin{align}
    M(\varepsilon) & = \left |\left \{\beta \in \ZZ^{d}_{\geq 0}: |\beta|=D, \left ( \prod_{j=1}^d j^{\beta_j} \right )^{-\alpha} \geq r_0(\Sigma)^D\varepsilon \right \}\right | \\
    & = \left |\left \{\beta \in \ZZ^{d}_{\geq 0}: |\beta|=D, \prod_{j=1}^d j^{\beta_j} \leq \dfrac{1}{r_0(\Sigma)^{\frac{D}{\alpha}}\varepsilon^{\frac{1}{\alpha}}} \right \}\right | \\
    & = \left |\left \{ (i_1, \dots, i_{D}): 1\leq i_1 \leq i_2 \leq \dots \leq i_{D} \leq d,  \prod_{j=1}^D i_{D}\leq \dfrac{1}{r_0(\Sigma)^{\frac{D}{\alpha}}\varepsilon^{\frac{\alpha}{\alpha}}} \right \} \right |
    \label{eq:mid_step_ordering_eigenvalues_1}
\end{align}
Now, let $X_D(L):= \{ |(i_1, \dots, i_{D}): 1\leq i_1 \leq \dots \leq i_{D} \leq d, \prod_{j=1}^D i_{D}\leq L | \}$. We can write the following recursion following \cite{T15}, Chapter I.3: 
\begin{align}
    X_D(L) & = \sum_{i_1=1}^{d} X_{D-1} \left (\left \lfloor \dfrac{L}{i_1} \right \rfloor \right ),
\end{align}
which we obtained just by fixing the first coordinate. We can then iterate this $D-1$ times to get: 
\begin{align}
    X_D(L) & = \sum_{i_1=1}^{d} \dots  \sum_{i_{D-1}=1}^{d} X_{1} \left (\left \lfloor \dfrac{L}{i_1 \cdots i_{D-1}} \right \rfloor \right ).
    \label{eq:mid_step_ordering_eigenvalues_2}
\end{align}
Note that $X_{1} \left (\left \lfloor \dfrac{L}{i_1 \cdots i_{D-1}} \right \rfloor \right )$ corresponds to the number of integers below this threshold, so $X_{1} \left (\left \lfloor \dfrac{L}{i_1 \cdots i_{D-1}} \right \rfloor \right ) = \left \lfloor \dfrac{L}{i_1 \cdots i_{D-1}} \right \rfloor$. We can replace this in Equation~\eqref{eq:mid_step_ordering_eigenvalues_2} to get: 
\begin{equation}
    X_D(L) = C L \mathrm{poly}\log(L). 
\end{equation}
Then, going back to Equation~\eqref{eq:mid_step_ordering_eigenvalues_1}, we obtain: 
\begin{equation}
    M(\varepsilon) = C \dfrac{\log(d)}{r_0(\Sigma)^{\frac{D}{\alpha}}\varepsilon^{\frac{1}{\alpha}}}.
    \label{eq:M_epsilon_order_D}
\end{equation}
Inverting this equation we get: 
\begin{equation}
    \varepsilon(M) = C \dfrac{M^{-\alpha}\log(d)}{r_0(\Sigma)^{D}}, 
\end{equation}
which is telling us that the $M$-th eigenvalue of order $ C \dfrac{M^{-\alpha}\log(d)}{r_0(\Sigma)^{D}}$. This is precisely the result in Corollary~\ref{cor:power_law_decay_monomial}.  

\subsection{Ordering the Spectrum for Finite-degree polynomials}

Now, we consider the more challenging problem where 
\begin{equation}
    k(x,x') = \sum_{k=0}^{D} h_{k} \langle x,x' \rangle^{k},
\end{equation}
and $x, x' \sim \gamma_d^{\alpha}$. In order to derive the correct ordering of the eigenvalues, spectral gaps will play a crucial role. To see this, we prove the following Lemma that characterizes when do inner product kernels have spectral gaps. 

\begin{lemma}[Spectral Gaps]
Let $\ell \in \NN$, and assume $\frac{1}{\ell + 2} \leq \alpha \leq \frac{1}{\ell + 1}$. Then, there exists a finite number of spectral gaps. In particular, between levels with multi-indices $\beta \in \ZZ^{d}_{\geq 0}$ with $|\beta|=j$ and $|\beta| = j+1$, for all $j \leq \ell$ we there is a spectral gap. 
\label{lemma:spectral_gaps}
\end{lemma}

\begin{proof}
By the structure we found in Proposition~\ref{prop:eigenvalues_operator}, for the power law setting, we have that the $\ell$-th level of eigenvalues of the kernel is separated from the $\ell+1$-th if and only if: 
\begin{equation}
     \dfrac{1}{r_0(\Sigma)^{\ell}d^{\alpha \cdot \ell}}  > \dfrac{1}{r_{0}(\Sigma)^{\ell+1}}. 
    \label{ineq:spectral_gap_1}
\end{equation}
Hence, from \cref{eq:asymptotics_effective dimension_1} we conclude that for $\alpha >1$ there are no spectral gaps in high dimensions, as $r_0(\Sigma) = O_d(1)$. However, when $\alpha \in [0,1)$, we have that $r_0(\Sigma) \asymp  d^{1- \alpha}$, so \cref{ineq:spectral_gap_1} becomes: 
\begin{equation}
    \dfrac{cd^{1-\alpha}}{d^{\alpha \cdot \ell}} > 1,
\end{equation}
from where we conclude that a necessary condition to have a spectral gap between levels $\ell$ and $k +1$ in high dimensions is: 
\begin{equation}
    (1-\alpha) \geq \alpha \cdot \ell \iff \alpha \leq \frac{1}{\ell+1}. 
\end{equation}
In particular, note that having a spectral gap between levels $\ell$ and $\ell+1$ implies a spectral gap between levels $j$ and $j+1$ for all $j \in [k]$. From here, we conclude that if we also have $\alpha \geq \frac{1}{\ell + 2}$, then there are no spectral gaps for $j \geq \ell + 1$. 
\end{proof}

\subsection*{The Order of the Eigenvalues - Proof of Proposition~\ref{prop:spectrum_power_law_final}}

We can now go back to our setting with 
\begin{equation}
    k(x,x') = \sum_{k=0}^{D} h_{k} \langle x,x' \rangle^{k},
\end{equation}
and $x, x' \sim \gamma_d^{\alpha}$.  From Lemma~\ref{lemma:spectral_gaps}, we know that we have two different cases: Either $\alpha \in [\frac{1}{\ell + 2}, \frac{1}{\ell +1})$ for some $\ell \in \NN$, or $\alpha \geq \frac{1}{2}$. In the first case, until we get to the eigenvalues $\lambda_{\beta}$ with $|\beta|\geq \ell+1$, there will be spectral gaps and the result will just follow from Corollary~\ref{cor:power_law_decay_monomial}. We study this in the following

\begin{lemma}
    Assume $\alpha \in [0,\frac{1}{D})$. Denote by $B_j = \binom{d-1 + j}{d-1}$, and $S_{L} = \sum_{j=0}^{L} B_j = \bigl [ \binom{L + d}{L} \bigr ]$, for $L \leq D$. Let $m \in [B_D]$, and assume there exists $j \leq D-1$ such that $S_j < m \leq S_{j+1}$. Then, there exists constants $C_1, C_2$, only depending on $\alpha$ and $j$ such that
    \[
     C_1 \cdot \dfrac{(m - S_j + 1)^{-\alpha}\log(d)}{r_0(\Sigma)^{j}} \leq \lambda_m \leq  C_2 \dfrac{(m - S_j + 1)^{-\alpha}\log(d)}{r_0(\Sigma)^{j}}
    \]
    \label{lemma:ordering_spectrum_finite_degree_spectral_gaps}
\end{lemma}

\begin{proof}
    Since $\alpha \in [0,\frac{1}{D})$, \cref{lemma:spectral_gaps} tells us that there are spectral gaps for all different levels in this kernel. More precisely, denoting the eigenvalues of the kernel by $\lambda_{\beta}$, for $\beta \in \ZZ^{d}_{\geq 0}$, with $|\beta|\leq D$, we will have that $|\beta| < |\gamma|$ implies $\lambda_{\beta} > \lambda_{\gamma}$. \\ 

    Now, consider our case $ S_j < m \leq S_{j+1}$. We will then have that the m-th eigenvalue $\lambda_m$ will belong to the level of eigenvalues with $|\beta| = j+1$. Hence, by Corollary~\ref{cor:power_law_decay_monomial}, we will get that there exists constants $C_1, C_2$ such that: 
    \[
    C_1 \cdot \dfrac{(m - S_j + 1)^{-\alpha}\log(d)}{r_0(\Sigma)^{j}} \leq \lambda_m \leq  C_2 \dfrac{(m - S_j + 1)^{-\alpha}\log(d)}{r_0(\Sigma)^{j}},
    \]
    which is what we wanted to conclude. 
\end{proof}

We can now ask ourselves: What happens when there is no spectral gap from a particular level? More precisely, assume $\ell < D$ and $\alpha \in [\frac{1}{\ell + 2} ,\frac{1}{\ell + 1})$, so that there are no spectral gaps for levels higher than $\ell$. Then, there will be a part of the eigenvalues that we will order with \cref{lemma:ordering_spectrum_finite_degree_spectral_gaps}, and after this we will have to count between different levels. We do this in the following 

\begin{lemma} Let $\ell \in \NN$, $\alpha \in [\frac{1}{\ell + 2} ,\frac{1}{\ell + 1})$, and $D >> L$. Let $\lambda_m$ denote the $m-th$ eigenvalue of the kernel $k(x,x') = \sum_{j=0}^D h_{j} \langle x,x'\rangle^k$, with $x,x' \sim \gamma_d^{\alpha}$. Then: 
\begin{itemize}
    \item \textbf{Spectral Gaps Sector:} If $\binom{d + j}{j} \leq m \leq \binom{d + j+1}{j+1}$, for $j \leq \ell$, then, there exists constants $C_1, C_2$, independent of $d$, such that: 
    \[
    C_1 \dfrac{\left ( m -\binom{d + j}{j}\right )^{-\alpha}}{r_0(\Sigma)^{j+1}}\mathrm{poly}\log(d) \leq \lambda_m \leq    C_1 \dfrac{\left ( m -\binom{d + j}{j}\right )^{-\alpha}}{r_0(\Sigma)^{j+1}}\mathrm{poly}\log(d).
    \]
    \item \textbf{Continuous Spectrum} If $m > \binom{d+\ell}{\ell}$, then there exists a strictly increasing sequence of numbers $a_\ell, \dots a_{D-1}$, such that $a_{j} = O(d^{j+1}\mathrm{poly}\log(d))$, and if $a_j \leq m \leq a_{j+1}$, then there exists constants $C_3, C_4$, independent of the dimension, such that: 
    \[
    C_3 \dfrac{\left ( m -a_j\right )^{-\alpha}}{r_0(\Sigma)^{j+1}}\mathrm{poly}\log(d) \leq \lambda_m \leq    C_4 \dfrac{\left ( m -a_{j}\right )^{-\alpha}}{r_0(\Sigma)^{j+1}}\mathrm{poly}\log(d).
    \]
\end{itemize}
\end{lemma}

\begin{proof}
    First, by a direct application of \cref{lemma:ordering_spectrum_finite_degree_spectral_gaps}, we get that for $j \leq \ell -1$, if $S_j < m \leq S_{j+1}$, then: 
    \begin{equation}
        C_1 \cdot \dfrac{(m - S_j + 1)^{-\alpha}\log(d)}{r_0(\Sigma)^{j}} \leq \lambda_m \leq  C_2 \dfrac{(m - S_j + 1)^{-\alpha}\log(d)}{r_0(\Sigma)^{j}}.
        \label{eq:general_ordering_spectral_gap_case}
    \end{equation}
    Now, assume $S_{\ell} < m \leq \binom{D +d}{D}$. We can split the eigenvalues of the kernel $\lambda_{\beta}$ into two groups: $A_1 := \{\lambda_{\beta}: |\beta|\leq \ell \}$, and $A_2 = \{\lambda_{\beta}: |\beta| \geq \ell +1 \}$. Equation~\eqref{eq:general_ordering_spectral_gap_case} gives an order in $A_1$, so we are left with ordering $A_2$, and $\lambda_m \in A_2$, as there is a spectral gap between levels $\ell$ and $\ell +1$. For this, we follow the same approach as we did in the proof of Corollary~\ref{cor:power_law_decay_monomial}.\\

    To order $A_2$, we note that all eigenvalues in $A_2$ are strictly less than $d^{\ell}$. Thus, we we can split it in the following way: 
    \begin{align}
        A_2 = \bigcup_{j=\ell}^{D-1} \underbrace{\{\lambda: \dfrac{1}{d^{j+1}} \leq \lambda \leq \dfrac{1}{d^{j}}  \}}_{A_{2,j}}.
    \end{align}
    Note that the sets $A_{2,j}$ partition $A_2$ into $D - \ell $ disjoint sets. Moreover, all the eigenvalues $\lambda_{\beta}$ in $A_{2,j}$ have $|\beta| \geq  j +1$. Then, for each $j \in \{ \ell, \dots, D-1\}$:
    \begin{align}
        |A_{2,j}| &= \left | \left \{\beta: \dfrac{1}{d^{j+1}} \leq \lambda_{\beta} \leq \dfrac{1}{d^j}  \right \} \right | \\
        & = \sum_{k \geq j+1} \left | \left \{\beta: |\beta|=k, \dfrac{1}{d^{j+1}} \leq \lambda_{\beta} \leq \dfrac{1}{d^j}  \right \} \right | \\
        & = \sum_{k \geq j+1} \left | \left \{\beta: |\beta|=k, \dfrac{r_0(\Sigma)^{k}}{d^{j+1}} \leq \left ( \prod_{a=1}^{d} a^{\beta_a} \right )^{-\alpha}\leq \dfrac{r_0(\Sigma)^k}{d^j}  \right \} \right | \\
        & =  \sum_{k \geq j+1} \left | \left \{\beta: |\beta|=k, \dfrac{ d^{\frac{j}{\alpha}} }{r_0(\Sigma)^{\frac{k}{\alpha}}} \leq \prod_{a=1}^{d} a^{\beta_a} \leq \dfrac{d^{\frac{j+1}{\alpha}}}{r_0(\Sigma)^{\frac{k}{\alpha}}}   \right \} \right |,
        \label{eq:A_2_j_cardinality}
    \end{align}
    and by applying the same argument as \cref{eq:mid_step_ordering_eigenvalues_2}, we conclude: 
    \begin{equation}
        |A_{2,j}|  = C\mathrm{poly} \log(d) \sum_{k \geq j+1}  \left (\dfrac{d^{\frac{j+1}{\alpha}}}{r_0(\Sigma)^{\frac{k}{\alpha}}}   - \dfrac{d^{\frac{j}{\alpha}}}{r_0(\Sigma)^{\frac{k}{\alpha}}}   \right ) = C \mathrm{poly}\log(d) \dfrac{d^{\frac{j+1}{\alpha}}}{r_0(\Sigma)^{\frac{j+1}{\alpha}}} =  Cd^{j + 1}\mathrm{poly}\log(d).
    \end{equation}
    Now, denote $a_{j} = \sum_{k=\ell}^{j}|A_{2,j}|$. Then, for $a_{j-1} \leq m \leq a_{j}$, we have that $\lambda_m \in A_{2,j}$. We can know order the eigenvalues inside $A_{2,j}$. We have:
    \begin{align}
        \left |\{\lambda \in A_{2,j}: \lambda \geq \varepsilon \} \right |  = \left | \left \{\beta \in \ZZ^{d}_{\geq 0}: \varepsilon \leq \lambda_{\beta} \leq \dfrac{1}{d^{j}} \right \} \right |,
    \end{align}
    and replicating \cref{eq:A_2_j_cardinality}, and then applying \cref{eq:mid_step_ordering_eigenvalues_2} we get: 
    \begin{equation}
        \left |\{\lambda \in A_{2,j}: \lambda \geq \varepsilon \} \right |  = C\mathrm{poly}\log(d) \sum_{k \geq j+1} \left (\dfrac{\varepsilon^{\frac{1}{\alpha}}}{r_0(\Sigma)^{\frac{k}{\alpha}}}   - \dfrac{d^{\frac{j}{\alpha}}}{r_0(\Sigma)^{\frac{k}{\alpha}}}   \right ) = C\mathrm{poly}\log(d) \dfrac{\varepsilon^{\frac{1}{\alpha}}}{r_0(\Sigma)^{\frac{j+1}{\alpha}}}.
    \end{equation}
    Then, inverting this relation we get that inside $A_{2,j}$, the $M$-th eigenvalue is
    \begin{equation}
        \lambda_M = C\mathrm{poly}\log(d) \dfrac{M^{-\alpha}}{r_0(\Sigma)^{j+1}}. 
    \end{equation}
    With this, we conclude the proof. 
\end{proof}

\newpage

\section{Generalization Error}
\label{appendix:Kernel_Matrix_Approximation}

The idea of this section is to compute the asymptotic generalization error of the following kernel $k: \RR^{d} \times \RR^d \to \RR$,
\begin{equation}
    k(x,x') = \sum_{k = 0}^L \xi_{k} \sum_{|\beta|=k} \binom{|\beta|}{\beta_1, \dots, \beta_d } \sigma_1^{\beta_1} \cdots \sigma_d^{\beta_d} He_{\beta}(z)He_{\beta}(z'), \xi_{k} \geq 0 \forall k \in [D],
    \label{eq:hermite_kernel_appendix}
\end{equation}
where $\sigma_i = \frac{i^{-\alpha}}{r_0(\Sigma)}$, for all $i \in [d]$, and $L$ is big.  

Note that the eigenvalues of this Kernel are of the same type as the ones in Proposition~\ref{prop:eigenvalues_operator}, with the difference that we changed the monomials of an inner product kernel to Hermite Polynomials. By the orthogonality of Hermite Polynomials, we have that:  
\begin{equation}
\int_{\RR^{d}} k(x,x') He_{\beta}(\Sigma^{-\frac{1}{2}}x') \gamma_d^{\alpha}(dx')= \xi_k \binom{|\beta|}{\beta_1, \dots, \beta_d} \lambda_1^{\beta_1} \cdot \lambda_d^{\beta-d} He_{\beta}(x),
\end{equation}
so we precisely know both the eigenvalues and eigenfunctions of our kernel. Having this, we will prove that the Assumption in \citep{GMMM20} and \citep{MMM22} in order to derive the asymptotic generalization error in high dimensions. \\

We will work in the setting where $n = O(d^{\kappa})$ for some $\kappa > 0$. We will denote $\KK \in \RR^{n \times n}$ as the empirical kernel matrix, and we assume $0 \leq \alpha <1$, and $n = Cd^{\kappa}$ for some generic constant $C$. \\ 

Now, we define the following sets of Assumptions on the eigenfunctions and eigenvalues of the kernel: 
\begin{assumption}[Kernel Concentration Properties]
 Let $k: \RR^{d} \times \RR^{d} \to \RR$ be a positive semi-definite kernel, and let $(\lambda_{d,i}, \psi_{i})_{i \geq 1}$ denote it's eigen-pairs. There exists integers $u(d)$ and $m(d)$, with $u(d) \geq m(d)$
        \begin{enumerate}
            \item (Hypercontractivity of finite Eigenspaces) For any $q \geq 1$, there exists $C$ such that all $h \in \mathrm{span}(\psi_i: \geq 1) $, 
            \[
            \| h\|_{L^{2q}} \leq C \|h \|_{L^2}. 
            \]
            \item (Properly Decaying Eigenvalues) There exists $\delta_0$ fixed, such that for all $d$ large enough, 
            \[
            n(d)^{2 +\delta_0} \leq \dfrac{(\sum_{j \geq u(d)
            +1} \lambda_{d,j}^{4})^{2}}{\sum_{j \geq u(d)+1} \lambda_{d,j}^{8}}, \text{ and }
            \]
            \[
            n(d)^{2 +\delta_0} \leq \dfrac{(\sum_{j \geq u(d)+1} \lambda_{d,j}^{2})^{2}}{\sum_{j \geq u(d)+1} \lambda_{d,j}^{4}}.
            \]
            \item (Concentration of diagonal elements) For all $x \sim \nu_{d}$, we have:
            \[
            \max_{i \in n(d)}\left | \EE_{x} \left [ k_{d, >m(d)}(x,x')^2 \right ] - \EE_{x,x'} \left [k_{d, m(d)} (x,x')^2 \right ]  \right | = o_d(1). 
            \]
            \[
            \max_{i \in n(d)}\left | k_{d, >m(d)}(x,x) - \EE_{x} \left [ k_{d, >m(d)}(x,x) \right ]\right | = o_d(1). 
            \]
            \end{enumerate}
    \label{assum:kernel_concentration}
\end{assumption}

\begin{assumption}[Eigenvalue Decay]
Let $k: \RR^{d} \times \RR^{d} \to \RR$ be a positive semi-definite kernel, and let $(\lambda_{d,i}, \psi_{i})_{i \geq 1}$ denote it's eigen-pairs.
\begin{enumerate}
            \item There exists $\delta_0 >0$, such that
            \[
            n(d)^{1 + \delta_0} \leq \dfrac{1}{\lambda_{d,m(d) + 1}^4} \sum_{k\geq m(d) + 1} \lambda_{d,k}^4,
            \]
            \[
            n(d)^{1 + \delta_0} \leq \dfrac{1}{\lambda_{d,m(d) + 1}^2} \sum_{k\geq m(d) + 1} \lambda_{d,k}^2.
            \]
            \item There exists $\delta_0 > 0$ such that
            \[
            m(d) \leq n(d)^{1- \delta_0}. 
            \]
        \end{enumerate}
     \label{assum:kernel_eigenvalue_decay}
\end{assumption}

Then, we can state the following Theorem from \cite{MMM22}: 
\begin{theorem}[Theorem 4 in \cite{MMM22}] 
Let $K: \RR^{d} \times \RR^{d} \to \RR$ be a positive semi-definite kernel, and let $(\lambda_{d,i}, \psi_{i})_{i \geq 1}$ denote it's eigen-pairs. Assume that $K$ satisfies Assumptions~\ref{assum:kernel_concentration} and \ref{assum:kernel_eigenvalue_decay}, and consider $\hat{f}$ to be the predictor of Kernel Ridge Regression with regularization parameter $\lambda >0$. Then, 
\[
\left | R(\hat{f}) - \| f^\star - \hat{f}^\mathrm{eff}_{\gamma^{\mathrm{eff}}}\|_{L^2} \right | =o_d(1),
\]
where: 
\begin{itemize}
    \item $\gamma^{\mathrm{eff}} = \lambda + \sum_{j \geq m(d)} \lambda_{d,j}$. 
    \item $\hat{f}^\mathrm{eff}_{\gamma^{\mathrm{eff}}} = \mathrm{arg} \min_{f} \{ \| f^\star - f\|_{L^2} + \dfrac{\gamma^{\mathrm{eff}}}{n} \| f\|_{\mathcal{H}}^2\}$.
\end{itemize}
\label{thm:big_theorem_montanari}
\end{theorem}

The idea will be to apply Theorem~\ref{thm:big_theorem_montanari} to our setting. For this, given out limited sample complexity $n = O_{d}(d^{\kappa})$ for some $\kappa>0$, we fix a small constant $\delta_0 > 0$ and define the subsets of multi-indices:
\begin{align*}
    \mathsf{High}(n) : = \left \{ \beta \in \ZZ^{d}_{\geq 0}: |\beta|\leq L,  \sigma_1^{\beta_1} \cdots \sigma_d^{\beta_d} \leq \dfrac{1}{d^{\kappa + \delta_0}} \right \},
\end{align*}   
and 
\begin{align*}
\mathsf{Low}(n) : = \left \{\beta \in \ZZ^{d}_{\geq 0}: \dfrac{\sigma_1^{\beta_1} \cdots \sigma_d^{\beta_d}}{\deff^{|\beta|}} > \dfrac{1}{d^{\kappa + \delta_0}} \right \}.
\end{align*}   
This induces a decomposition of the kernel spectrum $\lambda_{\beta}$ into high- and low-frequency sectors, corresponding to $\mathsf{High}(n)$ and $\mathsf{Low}(n)$. We will prove that this sets satisfy the Assumptions of Theorem~\ref{thm:big_theorem_montanari}. \\

Thus, we divide this section in two parts: In the first one, we will prove the Assumptions~\ref{assum:kernel_concentration}, and in the second one, we will prove Assumptions~\ref{assum:kernel_eigenvalue_decay}. 

\subsection{Proof of Assumptions~\ref{assum:kernel_concentration}}

In this section, we will prove that the kernel in \cref{eq:hermite_kernel_appendix} satisfies Assumptions~\ref{assum:kernel_concentration}, so the kernel matrix can be concentrated. We will prove everything for $m(d):=|\mathsf{Low}(n)|$. On the other hand, for $u(d)$ we will do the following: 

\begin{enumerate}
    \item First, we note that in the Proofs in \cite{MMM22} (particularly in the proof of Proposition 4, that proof the concentration of the off-diagonal of the empirical kernel matrix), it's also possible to fix a $u(d) \geq m(d)$, and concentrate a subset of the eigenvalues $\{ \lambda_{j}: j \geq u(d) \}$, as long as the set of eigenvalues that is left is finite. (Also note that, essentially, $u(d)$ corresponds to the eigenvalues for which a Frobenius bound of the operator norm works). 
    \item We will chose 
    \begin{equation}
        u(d):= \arg \min\{m: \lambda_m = \binom{|\beta|}{\beta} \sigma_1^{\beta_1} \cdots \sigma_d^{\beta_d}, \text{ for} |\beta| \geq 2D(\kappa) +1  \},
        \label{eq:choice_u(d)}
    \end{equation}
    and prove that:
    \[
    n(d)^{2 +\delta_0} \leq \dfrac{(\sum_{|\beta| \geq 2D(\kappa)+1
    } \lambda_{d,\beta}^{4})^{2}}{\sum_{{|\beta| \geq 2D(\kappa)+1}} \lambda_{d, \beta}^{8}}, \text{ and }
    \]
    \[
    n(d)^{2 +\delta_0} \leq \dfrac{(\sum_{|\beta| \geq 2D(\kappa)+1} \lambda_{d,\beta}^{2})^{2}}{\sum_{|\beta| \geq 2D(\kappa)+1} \lambda_{d,\beta}^{4}}.
    \]
\end{enumerate}

We will now prove each Assumption in \cref{assum:kernel_concentration} in three different lemmas. 

\begin{lemma}[Hypercontractivity of the Eigenspaces]
    The eigenfunctions of the kernel in \cref{eq:hermite_kernel_appendix} satisfy 1 in Assumption~\ref{assum:kernel_concentration}.
    \label{lemma:hypercontractivity_of_eigenspaces}
\end{lemma}

\begin{proof}
    Since the eigenfunctions of the kernel in \cref{eq:hermite_kernel_appendix} are polynomials, and the measure of the inputs is the Gaussian Measure, the Lemma is true by Gaussian Hypercontractivity (\cite{BLM13}, Corollary 5.21). 
\end{proof}

\begin{lemma}[Properly Decaying Eigenvalues]
 There exists $\delta_0$ fixed, such that for all $d$ large enough, 
            \[
            n(d)^{2 +\delta_0} \leq \dfrac{(\sum_{j \geq u(d)+1} \lambda_{d,j}^{2})^{2}}{\sum_{j \geq u(d)+1} \lambda_{d,j}^{4}},
            \]
            \[
            n(d)^{2 +\delta_0} \leq \dfrac{(\sum_{j \geq u(d)
            +1} \lambda_{d,j}^{4})^{2}}{\sum_{j \geq u(d)+1} \lambda_{d,j}^{8}}, \text{ and }
            \]
            for $m(d):=|\mathsf{Low}(n)|$, and $u(d) = \binom{d+ 2D(\kappa)  +1}{2D(\kappa)  +1}$. 
            \label{lemma:properly_decaying_eigenvalues}
\end{lemma}

\begin{proof}
    Recall in our setting $n = d^{\kappa}$. Then, for the choice of $u(d)$ in \cref{eq:choice_u(d)}, we have: 
    \begin{align}
        \sum_{|\beta| \geq 2D(\kappa)+1} \lambda_{d,\beta}^{4} &= \sum_{k \geq 2D(\kappa) + 1} \xi_k^2 \sum_{|\beta| = k} \binom{|\beta|}{\beta}^2 \sigma_1^{2\beta_1} \cdots \sigma_d^{2\beta_d} \\ 
        & \leq \max_{|\beta| \geq 2D(\kappa ) +1} \lambda_{\beta} \left (\sum_{k \geq 2D(\kappa) + 1} \xi_k \sum_{|\beta| = k} \binom{|\beta|}{\beta} \sigma_1^{\beta_1} \cdots \sigma_d^{\beta_d} \right ) \\
        & =  \max_{|\beta| \geq 2D(\kappa ) +1} \lambda_{\beta}  \left (\sum_{k \geq 2D(\kappa) + 1} \xi_k Tr(\Sigma)^{k} \right ) \\
        &= O\left ( \max_{|\beta| \geq 2D(\kappa ) +1} \lambda_{\beta}\right ). 
    \end{align}
    Then, since we have that: 
    \begin{equation}
        \max_{|\beta| \geq 2D(\kappa ) +1} \lambda_{\beta} = O\left ( \dfrac{1}{r_0(\Sigma)^{2D(\kappa) + 1}} \right ) = o_d(n^2). 
    \end{equation}
    With this, we can conclude the second inequality (as we see that $(\sum_{j \geq u(d)+1} \lambda_{d,j}^{2})^{2} =O(1)$. Proving the second inequality is analogous.  
\end{proof}

\subsection{Concentration of the Diagonal}

Before the third part fo Assumption~\ref{assum:kernel_concentration}, which concerns the concentration of diagonal elements, we will state the following useful Lemma. 

\begin{lemma}Let $p \geq 1$, and let $he_{p}(u)$ denote the $p-th$ normalized Hermite polynomial in $\RR$. Then: 
\[
he_{p}(u)^2 = \sum_{r=0}^{p}  C(p,r) he_{2p-2r}(u),
\]  
for some coefficients $C(p,r)$ that are $O_d(1)$ w.r.t the dimension. Doing a change of variables: 
\[
he_{p}(u)^2 = \sum_{r=0, p \equiv_2 r}^{p} C(p, r) he_{2r}(u),
\]
\label{lemma:hermite_squared_R}
\end{lemma}

\begin{proof}
    The proof is a direct application of the product formula of different Weiner Chaoses (\cite{NP12}, Theorem 2.7.1). 
\end{proof}

Note that, we can extend Lemma~\ref{lemma:hermite_squared_R} to Hermite polynomials in $\RR^d$, just by taking products. 
\begin{lemma}
Let $\beta \in \ZZ^{d}_{\geq 0}$, and let $He_{\beta}(z): = \prod_{a=1}^d he_{a_i}(z_i)$ . Then: 
\[
He_{\beta}(z)^2 = \sum_{\gamma \leq \beta: \gamma \equiv_2 \beta} C(\beta,\gamma) He_{2\gamma}(z),
\]  
for some constants $C(\beta,\gamma)$ that are $O(1)$ w.r.t the dimension. 
\label{lemma:hermite_squared_R_d}
\end{lemma}

Now, to prove the concentration of the diagonal Assumption, we need to prove that for all $x \sim \gamma_d^{\alpha}$, we have:
\begin{enumerate}
\item 
\[
\max_{i \in n(d)}\left | \EE_{x} \left [ k_{d, >m(d)}(x,x')^2 \right ] - \EE_{x,x'} \left [k_{d, m(d)} (x,x')^2 \right ]  \right | = o_d(1). 
\]
\item 
\[
\max_{i \in n(d)}\left | k_{d, >m(d)}(x,x) - \EE_{x} \left [ k_{d, >m(d)}(x,x) \right ]\right | = o_d(1). 
\]  
\end{enumerate} 
Note that:
\begin{align}
    \EE_{x} \left [ k_{d, >m(d)}(x,x')^2 \right ]  & = \EE_z \left [\left ( \sum_{\beta \in \mathsf{High}(n)} \xi_{|\beta|} \binom{|\beta|}{\beta_1, \dots, \beta_d } \sigma_1^{\beta_1} \cdots \sigma_d^{\beta_d} He_{\beta}(z)He_{\beta}(z') \right )^2\right ] \\
    & = \sum_{\beta \in \mathsf{High}(n)} \xi_{|\beta|}^2 \binom{|\beta|}{\beta_1, \dots, \beta_d }^2 \sigma_1^{2\beta_1} \cdots \sigma_d^{2\beta_d} He_{\beta}(z')^2.
\end{align}
And in the same way, we will have:
\begin{equation}
    k_{d, >m(d)}(x,x) = \sum_{\beta \in \mathsf{High}(n)} \xi_{|\beta|} \binom{|\beta|}{\beta_1, \dots, \beta_d } \sigma_1^{\beta_1} \cdots \sigma_d^{\beta_d} He_{\beta}(z')^2.
\end{equation}
We can then define the functions: 
\begin{align}
    F_1(x) & = \sum_{\beta \in \mathsf{High}(n)} \xi_{|\beta|}^2 \binom{|\beta|}{\beta_1, \dots, \beta_d }^2 \sigma_1^{2\beta_1} \cdots \sigma_d^{2\beta_d} He_{\beta}(z')^2, \\
    F_2(x) & = \sum_{\beta \in \mathsf{High}(n)} \xi_{|\beta|} \binom{|\beta|}{\beta_1, \dots, \beta_d } \sigma_1^{\beta_1} \cdots \sigma_d^{\beta_d} He_{\beta}(z')^2.
    \label{def_F_1_F_2} 
\end{align}

We will further decompose this functions int he following way: 
\begin{equation}
    F_1(x) = \underbrace{\sum_{\beta \in \mathsf{High}(n): |\beta |\leq D(\kappa)} \xi_{|\beta|}^2 \binom{|\beta|}{\beta_1, \dots, \beta_d }^2 \sigma_1^{2\beta_1} \cdots \sigma_d^{2\beta_d} He_{\beta}(z')^2}_{F_1^{\leq D(\kappa)}(x):=} +  \underbrace{\sum_{\beta \in \mathsf{High}(n): |\beta |> D(\kappa)} \xi_{|\beta|}^2 \binom{|\beta|}{\beta_1, \dots, \beta_d }^2 \sigma_1^{2\beta_1} \cdots \sigma_d^{2\beta_d} He_{\beta}(z')^2}_{F_1^{>D(\kappa)}(x):=}
    \label{eq:def_F_1_high}
\end{equation}
and analogously with $F_2$. Note that all eigenvalues associated to $\beta \in \ZZ^{d}_{\geq 0}$ with $|\beta| > D(\kappa)$ are less or equal than $r_0(\Sigma)^{-(D(\kappa ) +1)}$. Hence, we have that$\{\beta \in \mathsf{High}(n): |\beta |> D(\kappa)\} = \{ \beta: D(\kappa) + 1 \leq |\beta| \leq L \}$. This way: 
\begin{equation}
    F_1^{>D(\kappa)}(x) = \sum_{\beta : D(\kappa) + 1\leq |\beta | \leq L} \xi_{|\beta|}^2 \binom{|\beta|}{\beta_1, \dots, \beta_d }^2 \sigma_1^{2\beta_1} \cdots \sigma_d^{2\beta_d} He_{\beta}(z')^2,
\end{equation}
and the same holds for $F^{>D(\kappa)}_2(x)$. We can then concentrate $F_1^{>D(\kappa)}(x)$ and $F_2^{>D(\kappa)}(x)$. We do this in the following two Lemmas. 

\begin{lemma}[Concentration of $F_1$] Consider the function $F_1^{>D(\kappa)}$ defined in \cref{eq:def_F_1_high}. We have that: 
\[
\| F_1^{>D(\kappa)}(x) - \EE_{x}[F_1^{>D(\kappa)}(x)] \|_{L^2} = O \left ( \dfrac{C}{R_0(\Sigma)^{k}} \right ). 
\]
\label{lemma:concentration_F_1_high}
\end{lemma}

\begin{proof}
We can proceed as in Proposition 4 of \cite{MMM22}. Note that by Minkowski Inequality, we have that: 
\begin{align}
    \| F_1^{>D(\kappa)}(x) - \EE_{x}[F_1^{>D(\kappa)}(x)] \|_{L^2} = \| &\sum_{k=D(\kappa)+1}^{L} \xi_{k}\sum_{\beta \in \mathsf{High}(n): |\beta|=k}  \binom{|\beta|}{\beta_1, \dots, \beta_d }^2 \sigma_1^{2\beta_1} \cdots \sigma_d^{2\beta_d} (He_{\beta}(z')^2 -1) \|_{L^2} \\
    & \leq \sum_{k=D(\kappa)+1}^{L} \xi_{k}\sum_{\beta \in \mathsf{High}(n): |\beta|=k}  \binom{|\beta|}{\beta_1, \dots, \beta_d }^2 \sigma_1^{2\beta_1} \cdots \sigma_d^{2\beta_d} \|  (He_{\beta}(z')^2 -1)\| _{L^2}.
\end{align}
Then, since $ \|  (He_{\beta}(z')^2 -1)\| _{L^2} =O_d(1)$ by the triangular inequality, we get: 
\begin{equation}
    \| F_1^{>D(\kappa)}(x) - \EE_{x}[F_1^{>D(\kappa)}(x)] \|_{L^2}  \leq C \sum_{k=D(\kappa)+1}^{L} \xi_{k}^2\sum_{\beta \in \mathsf{High}(n): |\beta|=k}  \binom{|\beta|}{\beta_1, \dots, \beta_d }^2 \sigma_1^{2\beta_1} \cdots \sigma_d^{2\beta_d}. 
\end{equation}
Now, we can get rid of the squares in the binomial by bounding them by constants independent of $d$, and get: 
\begin{equation}
    \| F_1^{>D(\kappa)}(x) - \EE_{x}[F_1^{>D(\kappa)}(x)] \|_{L^2}  \leq C \sum_{k=D(\kappa)+1}^{L} \xi_{k}^2 \sum_{\beta \in \mathsf{High}(n): |\beta|=k}  \binom{|\beta|}{\beta_1, \dots, \beta_d } \sigma_1^{2\beta_1} \cdots \sigma_d^{2\beta_d}. 
\end{equation}
Note that the RHS corresponds exactly to powers of traces of $\Sigma^2$. We will then get: 
\begin{equation}
    \| F_1^{>D(\kappa)}(x) - \EE_{x}[F_1^{>D(\kappa)}(x)] \|_{L^2}  \leq C \sum_{k=D(\kappa)+1}^{L} \xi_{k}^2 Tr(\Sigma^2)^{k}. 
\end{equation}
Then we have that
\begin{align}
     Tr(\Sigma^2) = \dfrac{1}{r_0(\Sigma)^2}\sum_{j=1}^d i^{-2\alpha} = R_0(\Sigma),
\end{align}
by \cref{def:effective_dim}. Hence, we obtain: 
\begin{equation}
     \| F_1^{>D(\kappa)}(x) - \EE_{x}[F_1^{>D(\kappa)}(x)] \|_{L^2}  \leq \dfrac{C}{R_0(\Sigma)^{D(\kappa) +1}},
\end{equation}
so we conclude. 
\end{proof}

\begin{lemma}[Concentration of $F_2^{>D(\kappa)}$]Consider the function $F_1^{>D(\kappa)}$ defined in \cref{eq:def_F_1_high}. We have that: 
\[
\| F_1^{>D(\kappa)} - \EE_x \left [ F_1^{>D(\kappa)}(x)\right ]\|_{L^2} \leq\dfrac{C}{R_0(\Sigma)^{\frac{D(\kappa) +1}{2}}}
\]
\label{lemma:concentration_F_2_high}
\end{lemma}

\begin{proof}
    By definition we have that:
    \begin{equation}
        F_2^{>D(\kappa)}(x) = \sum_{k=D(\kappa) +1}^{L}\xi_{k} 
        \sum_{|\beta| = k} \xi_{|\beta|} \binom{|\beta|}{\beta_1, \dots, \beta_d }\sigma_1^{\beta_1} \cdots \sigma_d^{\beta_d} He_{\beta}(z')^2.
    \end{equation}
    From here, we note that the argument we used in \cref{lemma:concentration_F_1_high} will not work, as the sum of the coefficients will be $O(1)$. Therefore, we will apply \cref{lemma:hermite_squared_R_d} to get: 
    \begin{equation}
        F_2^{>D(\kappa)}(x) = \sum_{k=D(\kappa) +1}^{L}\xi_{k} 
        \sum_{|\beta| = k} \binom{|\beta|}{\beta_1, \dots, \beta_d }\sigma_1^{\beta_1} \cdots \sigma_d^{\beta_d} \sum_{\gamma \leq \beta: \gamma \equiv_2 \beta} C(\beta,\gamma) He_{2\gamma}(z),
    \end{equation}
    for some constants $C(\beta,\gamma)$ uniformly bounded on $d$.  Exchanging the sums we get: 
    \begin{equation}
        F_2^{>D(\kappa)}(x) = \sum_{|\gamma| \leq L} He_{2\gamma}(z) \underbrace{\sum_{k=D(\kappa) +1}^{L}\xi_{k} 
        \sum_{|\beta| = k: \beta \geq \gamma, \gamma \equiv_2 \beta} \binom{|\beta|}{\beta_1, \dots, \beta_d }\sigma_1^{\beta_1} \cdots \sigma_d^{\beta_d}}_{S_{\gamma}}.
    \end{equation}
    We then get the Hermite decomposition of $F_2^{>D(\kappa)}$: 
     \begin{equation}
        F_2^{>D(\kappa)}(x) = \sum_{|\gamma| \leq L}  S_{\gamma}He_{2\gamma}(z)
    \end{equation}
    In particular, we have that:
    \begin{equation}
    \| F_2^{>D(\kappa)}(x)\|_{L^2}^2 = \sum_{|\gamma| \leq L}  S_{\gamma}^2.
    \label{eq:F_2_higher_degree_energy_1}
    \end{equation}
    Note that we can re-write the expression of $S_{\gamma}$ by re-indexing the sum in the interior. More precisely, we have: 
    \begin{align}
         S_{\gamma} & = \sum_{k=D(\kappa) +1}^{L}\xi_{k} 
        \sum_{|\beta| = k: \beta \geq \gamma, \gamma \equiv_2 \beta} \binom{|\beta|}{\beta_1, \dots, \beta_d }\sigma_1^{\beta_1} \cdots \sigma_d^{\beta_d} \\
        & = \sum_{k=D(\kappa) +1}^{L} \mathbf{1}_{k \equiv_2 \gamma}\xi_{k} 
        \sum_{\zeta \in \ZZ^{d}_{\geq 0}: |\gamma + 2\zeta| = k} \binom{|\gamma + 2\zeta|}{(\gamma + 2\zeta)_1, \dots, (\gamma + 2\zeta)_d }\sigma_1^{(\gamma + 2\zeta)_1} \cdots \sigma_d^{(\gamma + 2\zeta)_d} \\
        & = \sum_{k=D(\kappa) +1}^{L} \mathbf{1}_{k \equiv_2 \gamma}\xi_{k} 
        \sigma_1^{\gamma_1} \cdots \sigma_d^{\gamma_d} \sum_{\zeta \in \ZZ^{d}_{\geq 0}: 2|\zeta| = k - |\gamma|} \binom{|\gamma + 2\zeta|}{(\gamma + 2\zeta)_1, \dots, (\gamma + 2\zeta)_d }\sigma_1^{(2\zeta)_1} \cdots \sigma_d^{(2\zeta)_d}.
    \end{align}
    Then, by the same argument we used in the proof of \cref{lemma:concentration_F_1_high}, up to constants that don't depend on $d$, we will have: 
    \begin{align}
        S_{\gamma} & = \sigma_1^{\gamma_1} \cdots \sigma_d^{\gamma_d}\sum_{k=D(\kappa) +1}^{L} \mathbf{1}_{k \equiv_2 \gamma}\xi_{k} \dfrac{1}{R_0(\Sigma)^{\frac{k-|\gamma|}{2}}} 
        \label{eq:order_hermite_coefficient_F_2_higher}
    \end{align}
    Then, going back to \cref{eq:F_2_higher_degree_energy_1}, we can replace \cref{eq:order_hermite_coefficient_F_2_higher} to get:  
    \begin{align}
        \| F_1^{>D(\kappa)}\|_{L^2}^2 &= \sum_{|\gamma| \leq L}  S_{\gamma}^2 \\
        & = \sum_{|\gamma| \leq D(\kappa)} S_{\gamma}^2 + \sum_{D(\kappa ) + 1|\gamma| \leq D(\kappa)} S_{\gamma}^2  \\
        & = O \left ( \sum_{|\gamma| \leq D(\kappa)} \dfrac{\sigma_1^{2\gamma_1} \cdots \sigma_d^{2\gamma_d}}{R_0(\Sigma)^{D(\kappa) + 1 -|\gamma|}}  + \sum_{D(\kappa ) + 1\leq |\gamma| \leq L} \sigma_1^{2\gamma_1} \cdots \sigma_d^{2\gamma_d} \right ) \\
        & = O \left ( \sum_{|\gamma| \leq D(\kappa)} \dfrac{\sigma_1^{2\gamma_1} \cdots \sigma_d^{2\gamma_d}}{R_0(\Sigma)^{D(\kappa) + 1 -|\gamma|}}  + \sum_{D(\kappa ) + 1 \leq |\gamma| \leq L} \sigma_1^{2\gamma_1} \cdots \sigma_d^{2\gamma_d} \right ).
        \label{eq:F_2_higher_degree_energy_2}
    \end{align}
    Then by the same arguments that we used in \cref{lemma:concentration_F_1_high}, we can group terms according to the value of $|\gamma|$, and get: 
    \begin{equation}
        \| F_1^{>D(\kappa)}\|_{L^2}^2  = O(\dfrac{1}{R_0(\Sigma)^{D(\kappa) +1 }}),
    \end{equation}
    and then: 
    \begin{equation}
         \| F_1^{>D(\kappa)}\|_{L^2} =  O \left (\dfrac{1}{R_0(\Sigma)^{\frac{D(\kappa) +1}{2}}} \right ).
    \end{equation}
    Since $F_1^{>D(\kappa)}$ and $\EE_x \left [ F_1^{>D(\kappa)}(x)\right ]$ are greater than $0$, we can conclude that:
    \begin{equation}
        \| F_1^{>D(\kappa)} - \EE_x \left [ F_1^{>D(\kappa)}(x)\right ]\|_{L^2} \leq  \| F_1^{>D(\kappa)}\|_{L^2}  \leq\dfrac{C}{R_0(\Sigma)^{\frac{D(\kappa) +1}{2}}}. 
    \end{equation}
\end{proof}

We are now lest with concentrating $F_1^{\leq D(\kappa)}(x)$ and $F_2^{\leq D(\kappa)}(x)$. We recall their definitions: 
\begin{align}
    F_1^{>D(\kappa)}(x) & = \sum_{\beta \in \mathsf{High}(n): |\beta |\leq D(\kappa)} \xi_{|\beta|}^2 \binom{|\beta|}{\beta_1, \dots, \beta_d }^2 \sigma_1^{2\beta_1} \cdots \sigma_d^{2\beta_d} He_{\beta}(z')^2 \\
    F_2^{>D(\kappa)}(x) & = \sum_{\beta \in \mathsf{High}(n): |\beta |\leq D(\kappa)} \xi_{|\beta|} \binom{|\beta|}{\beta_1, \dots, \beta_d } \sigma_1^{\beta_1} \cdots \sigma_d^{\beta_d} He_{\beta}(z')^2
\end{align}
The idea will be to replicate the proof of \cref{lemma:concentration_F_2_high}, but since this time we are not able to express the sums in terms of the effective dimensions, we will have to use the special structure we have for $\sigma_j$, which have a power-law decay. In particular, we will need \cref{cor:power_law_decay_monomial}. 

\begin{lemma}Consider the functions $F_1^{\leq D(\kappa)}$ and $F_2^{\leq D(\kappa)}$ defined in \cref{eq:def_F_1_high}. We have: 
\[
  \|  F_1^{>D(\kappa)} - \EE_{x}[F_1^{>D(\kappa)} (x)] \|_{L^2}^2 = O \left (\dfrac{\mathrm{poly}\log(d)}{d^{\kappa + \delta_0}} \right )
\]
\[
\|  F_2^{>D(\kappa)} - \EE_{x}[F_2^{>D(\kappa)} (x)] \|_{L^2}^2 = O \left (\dfrac{\mathrm{poly}\log(d)}{\sqrt{d^{\kappa + \delta_0}}} \right )
\]
\label{lemma:concentration_F_1_F_2_low}
\end{lemma}

\begin{proof}
    We will only do the proof for $F_2^{\leq D(\kappa)}$, as it is harder. The proof for $F_1^{\leq D(\kappa)}$ is easier as the coefficients are smaller.  
    
    First, we can apply \cref{lemma:hermite_squared_R_d} to re-write $F_2^{\leq D(\kappa)}$: 
    \begin{align}
         F_2^{>D(\kappa)}(x) & = \sum_{\beta \in \mathsf{High}(n): |\beta |\leq D(\kappa)} \xi_{|\beta|} \binom{|\beta|}{\beta_1, \dots, \beta_d } \sigma_1^{\beta_1} \cdots \sigma_d^{\beta_d} He_{\beta}(z')^2 \\
         & =  \sum_{\beta \in \mathsf{High}(n): |\beta |\leq D(\kappa)} \xi_{|\beta|} \binom{|\beta|}{\beta_1, \dots, \beta_d } \sigma_1^{\beta_1} \cdots \sigma_d^{\beta_d} \sum_{\gamma \leq \beta: \gamma \equiv_2 \beta} C(\beta,\gamma) He_{2\gamma}(z) \\
         & = \sum_{|\gamma| \leq D(\kappa)} He_{2\gamma}(z) \underbrace{\sum_{\beta \in \mathsf{High}(n): \beta \geq \gamma, \gamma \equiv_2 \beta, |\beta|\leq D(\kappa)} \xi_{|\beta|}
         \binom{|\beta|}{\beta_1, \dots, \beta_d }\sigma_1^{\beta_1} \cdots \sigma_d^{\beta_d}}_{S_{\gamma}},
         \label{eq:def_S_gamma_lower_F}
    \end{align}
    where in the last line we exchanged the sums. Now, let's study the coefficients $S_{\gamma}$ for a moment. Recall that: 
    \begin{equation}
        \beta \in \mathsf{High}(n) \iff |\beta| \leq L, \text{ and } \sigma_1^{\beta_1} \cdots \sigma_d^{\beta_d} \leq \dfrac{1}{d^{\kappa + \delta_0}}. 
    \end{equation}
    Now, if we take $\beta \in \ZZ^{d}_{\geq 0}$ with $|\beta| \leq \lfloor \kappa \rfloor $, then we will have that: 
    \begin{equation}
        \sigma_1^{\beta_1} \cdots \sigma_d^{\beta_d} \geq \frac{1}{d^{\lfloor \kappa \rfloor}},
    \end{equation}
    as the minimum value we could have corresponds to taking $\beta_d = \lfloor \kappa \rfloor$. Since we assume $\kappa \not = \lfloor \kappa \rfloor$, we have that, for all $\beta \in \ZZ^{d}_{\geq 0}$ with $|\beta| \leq \lfloor \kappa \rfloor $, $\beta \in \mathsf{Low}(n)$. Hence, we conclude that 
    \begin{equation}
        \mathsf{High}(n) \subseteq \{\beta \in \ZZ^{d}_{\geq 0}: \lfloor \kappa \rfloor + 1 \leq |\beta| \leq L\}.
    \end{equation}
    Then, we can decompose $S_{\gamma}$ in \cref{eq:def_S_gamma_lower_F} by the degrees of $\beta \in \mathsf{High}(n)$: 
    \begin{equation}
        S_{\gamma} = \sum_{k=\lfloor \kappa \rfloor +1}^{L} \sum_{\beta \in \mathsf{High}(n): |\beta|=k} \mathbf{1}_{\substack{\beta \geq \gamma,\\ \gamma \equiv_2 \beta}} \xi_{|\beta|}
         \binom{|\beta|}{\beta_1, \dots, \beta_d }\sigma_1^{\beta_1} \cdots \sigma_d^{\beta_d}.
    \end{equation}
    We can then re-index the inner sum
    \begin{equation}
        S_{\gamma} = \sum_{k=\lfloor \kappa \rfloor +1}^{L} \mathbf{1}_{k\equiv_2 |\gamma|} \sum_{\gamma + 2\zeta \in \mathsf{High}(n): |\gamma| + 2|\zeta|=k} 
         \binom{|\gamma| + 2|\zeta|}{(\gamma + 2\zeta)_1, \dots, (\gamma + 2\zeta)_d }\sigma_1^{(\gamma + 2\zeta)_1} \cdots \sigma_d^{(\gamma + 2\zeta)_d},
    \end{equation}
    and re-write it: 
    \begin{equation}
        S_{\gamma} =\sum_{k=\lfloor \kappa \rfloor +1}^{L} \mathbf{1}_{k\equiv_2 |\gamma|} \sigma_1^{\gamma_1} \cdots \sigma_d^{\gamma_d} \sum_{|\zeta| = \frac{k -|\gamma|}{2}} \mathbf{1}_{\gamma + 2\zeta| \in \mathsf{High}(n)}
         \binom{|\gamma| + 2|\zeta|}{(\gamma + 2\zeta)_1, \dots, (\gamma + 2\zeta)_d }\sigma_1^{2\zeta_1} \cdots \sigma_d^{2\zeta_d}. 
    \end{equation}
    Then, by bounding the binomial coefficients (with constants that don't depend on $d$) we get: 
    \begin{equation}
        S_{\gamma} = O \left ( \sum_{k=\lfloor \kappa \rfloor +1}^{L} \mathbf{1}_{k\equiv_2 |\gamma|} \sigma_1^{\gamma_1} \cdots \sigma_d^{\gamma_d} \sum_{|\zeta| = \frac{k -|\gamma|}{2}} \mathbf{1}_{\gamma + 2\zeta| \in \mathsf{High}(n)}
         \binom{2|\zeta|}{2\zeta_1, \dots, 2\zeta_d }\sigma_1^{2\zeta_1} \cdots \sigma_d^{2\zeta_d} \right ).
         \label{eq:S_gamma_high_1}
    \end{equation}
    We could now hope to proceed the same way we did in \cref{lemma:concentration_F_2_high}. However, the indicator $\mathbf{1}_{\gamma + 2\zeta| \in \mathsf{High}(n)}$ does not allow it. Hence, we will have to do something else. Note that, by definition of $\mathsf{High}(n)$:
    \begin{align}
        \gamma + 2\zeta \in \mathsf{High}(n) & \iff \sigma_1^{\gamma_1 + 2\zeta_1} \cdots \sigma_1^{\gamma_d + 2\zeta_d}\leq \dfrac{1}{d^{\kappa + \delta_0}} \\
        & \iff \sigma_1^{ 2\zeta_1} \cdots \sigma_1^{2\zeta_d}\leq \dfrac{\sigma_1^{-\gamma_1} \cdots \sigma_1^{-\gamma_d}}{d^{\kappa + \delta_0}} \\
        & \iff \sigma_1^{ \zeta_1} \cdots \sigma_1^{\zeta_d}\leq \sqrt{\dfrac{\sigma_1^{-\gamma_1} \cdots \sigma_1^{-\gamma_d}}{d^{\kappa + \delta_0}}}.
        \label{eq:condition_gamma_d_zeta_high}
    \end{align}
    Now, by \cref{cor:power_law_decay_monomial} we know that, within the level $|\beta|=j$, we have $B_j :=\binom{d-1 + j}{d-1}$ eigenvalues, which we can order obtaining $\lambda_{j,1}, \cdots, \lambda_{j,B_j}$, with
    \begin{equation}
        \lambda_{j, m} = C\dfrac{m^{-\alpha}\mathrm{poly}\log(d)}{r_0(\Sigma)^{j}}. 
        \label{eq:order_eigenvalues_within_level_k}
    \end{equation}
    Then, replacing \cref{eq:condition_gamma_d_zeta_high} and \cref{eq:order_eigenvalues_within_level_k} in \cref{eq:S_gamma_high_1}: 
    \begin{equation}
        S_{\gamma} = O \left (\sum_{k=\lfloor \kappa \rfloor +1}^{L} \mathbf{1}_{k\equiv_2 |\gamma|} \sigma_1^{\gamma_1} \cdots \sigma_d^{\gamma_d} \sum_{m=1}^{B_{\frac{k -|\gamma|}{2}}} \mathbf{1}_{\left\{ \lambda_{\frac{k -|\gamma|}{2}, m} \leq \sqrt{ \frac{\sigma_1^{-\gamma_1} \cdots \sigma_d^{-\gamma_d}}{d^{\kappa + \delta_0}} } \right\}}
        \lambda_m^2 \right ).
        \label{eq:mid_step_hermite_coefficients_S_gamma}
    \end{equation}
    We can now re-write the indicator function in order to know what is the minimum value of $m$ in the inner sum: 
    \begin{align}
        \lambda_{\frac{k -|\gamma|}{2}, m} \leq \sqrt{ \frac{\sigma_1^{-\gamma_1} \cdots \sigma_d^{-\gamma_d}}{d^{\kappa + \delta_0}} } &\iff  \dfrac{m^{-\alpha}C\mathrm{poly}\log(d)}{r_0(\Sigma)^{\frac{k -|\gamma|}{2}}} \leq\sqrt{ \frac{\sigma_1^{-\gamma_1} \cdots \sigma_d^{-\gamma_d}}{d^{\kappa + \delta_0}} }  \\
        & \iff m^{\alpha} \geq \dfrac{C\mathrm{poly}\log(d) d^{\frac{\kappa + \delta_0}{2}} \sigma_1^{\frac{\gamma_1}{2}} \cdots \sigma_d^{\frac{\gamma_d}{2}} }{r_0(\Sigma)^{\frac{k -|\gamma|}{2}}} \\
        & \iff m \geq \left ( \dfrac{C\mathrm{poly}\log(d) d^{\frac{\kappa + \delta_0}{2}} \sigma_1^{\frac{\gamma_1}{2}} \cdots \sigma_d^{\frac{\gamma_d}{2}} }{r_0(\Sigma)^{\frac{k -|\gamma|}{2}}}\right )^{\frac{1}{\alpha}}.
    \end{align}
    We then define:
    \begin{equation}
        \mathsf{Min}(\frac{k -|\gamma|}{2}; \gamma) : = \left ( \dfrac{C\mathrm{poly}\log(d) d^{\frac{\kappa + \delta_0}{2}} \sigma_1^{\frac{\gamma_1}{2}} \cdots \sigma_d^{\frac{\gamma_d}{2}} }{r_0(\Sigma)^{\frac{k -|\gamma|}{2}}}\right )^{\frac{1}{\alpha}}.
        \label{eq:def_Min_k_gamma}
    \end{equation}
    Note that the fact that we only knew $\lambda_m$ up to constants will not matter, as we will only need the order of the minimum $m$, not the exact one. Going back to \cref{eq:mid_step_hermite_coefficients_S_gamma} we obtain: 
    \begin{equation}
        S_{\gamma} = O \left (\sigma_1^{\gamma_1} \cdots \sigma_d^{\gamma_d} \sum_{k=\lfloor \kappa \rfloor +1}^{L} \mathbf{1}_{k\equiv_2 |\gamma|}  \sum_{m=  \mathsf{Min}(\frac{k -|\gamma|}{2}; \gamma)}^{B_{\frac{k -|\gamma|}{2}}}
        \lambda_m^2 \right ).
        \label{eq:S_gamma_bound_2}
    \end{equation}
    Now, by \cref{eq:order_eigenvalues_within_level_k}: 
    \begin{align}
         \sum_{m=  \mathsf{Min}(\frac{k -|\gamma|}{2}; \gamma)}^{B_{\frac{k -|\gamma|}{2}}}
        \lambda_m^2 = O\left ( \dfrac{C \mathrm{poly}\log(d)}{r_0(\Sigma)^{k -|\gamma|}}  \sum_{m=  \mathsf{Min}(\frac{k -|\gamma|}{2}; \gamma)}^{B_{\frac{k -|\gamma|}{2}}} m^{-2\alpha} \right ).
    \end{align}
    For the inner sum, we bound $\frac{1}{m^{2\alpha}} \leq \frac{1}{\mathsf{Min}(\frac{k -|\gamma|}{2}; \gamma)^{\alpha}} \cdot \frac{1}{m^{\alpha}}$, and get: 
    \begin{align}
         \sum_{m=  \mathsf{Min}(\frac{k -|\gamma|}{2}; \gamma)}^{B_{\frac{k -|\gamma|}{2}}}
        \lambda_m^2 & = O\left ( \dfrac{C \mathrm{poly}\log(d)}{r_0(\Sigma)^{k -|\gamma|}\mathsf{Min}(\frac{k -|\gamma|}{2}; \gamma)^{\alpha} }  \sum_{m=  \mathsf{Min}(\frac{k -|\gamma|}{2}; \gamma)}^{B_{\frac{k -|\gamma|}{2}}} m^{-\alpha} \right ) \\
        & = O \left ( \dfrac{C \mathrm{poly}\log(d)}{r_0(\Sigma)^{k -|\gamma|}\mathsf{Min}(\frac{k -|\gamma|}{2}; \gamma)^{\alpha} } B_{\frac{k -|\gamma|}{2}}^{1-\alpha}\right ).
    \end{align}
    Recall that $B_{\frac{k -|\gamma|}{2}} = \binom{d-1 + {\frac{k -|\gamma|}{2}}}{d-1} = O (d^{\frac{k -|\gamma|}{2}}) $. Therefore, we have that $B_{\frac{k -|\gamma|}{2}}^{1-\alpha} = r_0(\Sigma)^{\frac{k -|\gamma|}{2}}$. Hence: 
    \begin{equation}
        \sum_{m=  \mathsf{Min}(\frac{k -|\gamma|}{2}; \gamma)}^{B_{\frac{k -|\gamma|}{2}}}
        \lambda_m^2  = O \left ( \dfrac{C \mathrm{poly}\log(d)}{r_0(\Sigma)^{\frac{k -|\gamma|}{2}}\mathsf{Min}(\frac{k -|\gamma|}{2}; \gamma)^{\alpha} }\right ).
    \end{equation}
    And recalling the definition of $\mathsf{Min}(\frac{k -|\gamma|}{2}; \gamma)$ in \cref{eq:def_Min_k_gamma} we get: 
    \begin{equation}
         \sum_{m=  \mathsf{Min}(\frac{k -|\gamma|}{2}; \gamma)}^{B_{\frac{k -|\gamma|}{2}}}
        \lambda_m^2  = O \left ( \dfrac{C \mathrm{poly}\log(d)}{d^{\frac{\kappa + \delta_0}{2}} \sigma_1^{\frac{\gamma_1}{2}} \cdots \sigma_d^{\frac{\gamma_d}{2}}}\right ).
    \end{equation}
    Replacing this in \cref{eq:S_gamma_bound_2}: 
    \begin{equation}
         S_{\gamma} = O \left ( \dfrac{\mathrm{poly}\log(d)\sqrt{\sigma_1^{\gamma_1} \cdots \sigma_d^{\gamma_d}}}{d^{\frac{\kappa + \delta_0}{2}}}\right ).
    \end{equation}
    We can now go all the way back to \cref{eq:def_S_gamma_lower_F}, to get: 
    \begin{align}
        \|  F_2^{>D(\kappa)}\|_{L^2}^2 &= \sum_{|\gamma|\leq 2D(\kappa) } S_{\gamma}^2 \\
        & = O \left ( \dfrac{\mathrm{poly}\log(d)}{d^{\kappa + \delta_0}}\sum_{|\gamma|\leq 2D(\kappa) } \sigma_1^{\gamma_1} \cdots \sigma_d^{\gamma_d} \right ). 
    \end{align}
    Note that the sum of the right is $O_d(1)$ (because of the normalization of the eigenvalues). Consequently: 
    \begin{equation}
         \|  F_2^{>D(\kappa)}\|_{L^2}^2 =  O \left ( \dfrac{\mathrm{poly}\log(d)}{d^{\kappa + \delta_0}} \right ).
    \end{equation}
    We conclude by noting that: 
    \begin{equation}
     \|  F_2^{>D(\kappa)} - \EE_{x}[F_2^{>D(\kappa)} (x)] \|_{L^2}^2 \leq \| F_2^{>D(\kappa)}\|_L^2,
    \end{equation}
    so 
    \begin{equation}
        \|  F_2^{>D(\kappa)} - \EE_{x}[F_2^{>D(\kappa)} (x)] \|_{L^2}^2 = O \left (\dfrac{\mathrm{poly}\log(d)}{\sqrt{d^{\kappa + \delta_0}}} \right ).
    \end{equation}
\end{proof}

We can now put \cref{lemma:concentration_F_1_high}, \cref{lemma:concentration_F_2_high}, and \cref{lemma:concentration_F_1_F_2_low} together to conclude the concentration of the diagonal. 

\begin{lemma}[Concentration of the diagonal matrices]
    Let $n = O(d^{\kappa})$. Then, under the assumptions of \cref{thm:generalization_error}, with high probability we have:  
    \[
\max_{i \in n(d)}\left | \EE_{x} \left [ k_{d, >m(d)}(x,x')^2 \right ] - \EE_{x,x'} \left [k_{d, m(d)} (x,x')^2 \right ]  \right | = o_d(1). 
\]
and 
\[
\max_{i \in n(d)}\left | k_{d, >m(d)}(x,x) - \EE_{x} \left [ k_{d, >m(d)}(x,x) \right ]\right | = o_d(1). 
\]  

\end{lemma}

\begin{proof}
    We will only do the second one, as both of them are analogous. First, in expectation we have: 
    \begin{align}
        \EE \left [ \max_{i \in n(d)}\left | \EE_{x} \left [ k_{d, >m(d)}(x,x')^2 \right ] - \EE_{x,x'} \left [k_{d, m(d)} (x,x')^2  \right ] \right | \right ] & \leq \EE \left [ \max_{i \in n(d)}\left | F_1(x) - \EE \left [ F_1(x) \right ] \right | \right ],
    \end{align} 
    with $F_1$ defined in \cref{eq:def_F_1_high}. Then, by Jensen's Inequality: 
    \begin{align}
        \EE \left [ \max_{i \in n(d)}\left | \EE_{x} \left [ k_{d, >m(d)}(x,x')^2 \right ] - \EE_{x,x'} \left [k_{d, m(d)} (x,x')^2  \right ] \right | \right ] & \leq \EE \left [ \max_{i \in n(d)}\left | F_1(x_i) - \EE \left [ F_1(x_i) \right ] \right |^2 \right ]^{\frac{1}{2}} \\
        & \leq \EE \left [ \sum_{i=1}^{n} \left | F_1(x_i) - \EE \left [ F_1(x_i) \right ] \right |^2 \right ]^{\frac{1}{2}} \\
        & \leq \sqrt{n} \| F_1(x_i) - \EE \left [ F_1(x_i) \right ]\|_{L^2}. 
    \end{align}
    Denote
    \[
    (\star) = \EE \left [ \max_{i \in n(d)}\left | \EE_{x} \left [ k_{d, >m(d)}(x,x')^2 \right ] - \EE_{x,x'} \left [k_{d, m(d)} (x,x')^2  \right ]  \right | \right ].
    \]
    Then, by triangular inequality we have: 
    \begin{align}
        (\star ) \leq \sqrt{n} \| F_1^{\leq D(\kappa)}(x_i) - \EE \left [ F_1(x_i)^{\leq D(\kappa)} \right ]\|_{L^2} + \sqrt{n} \| F_1^{>D(\kappa)}(x_i) - \EE \left [ F_1^{>D(\kappa)}(x_i) \right ]\|_{L^2}.
    \end{align}
    Now we apply Lemmas~\ref{lemma:concentration_F_1_high} and \ref{lemma:concentration_F_1_F_2_low} to get: 
    \begin{equation}
        (\star ) = O \left (\mathrm{poly}\log(d) \sqrt{\dfrac{n}{d^{\kappa + \delta_0}}}  + \mathrm{poly}\log(d) \sqrt{\dfrac{n}{R_0(\Sigma)^{D(\kappa) +1 }}}\right ).
    \end{equation}
    Then, since $n = O(d^{\kappa})$, the first term is negligible. For the second one, by Lemma 1 in \cite{BLLT20}, $R_0(\Sigma) \geq r_0(\Sigma)$, so we get: 
    \begin{equation}
        \sqrt{\dfrac{n}{R_0(\Sigma)^{D(\kappa) +1 }}} \leq \sqrt{\dfrac{d^{\kappa}}{r_0(\Sigma)^{D(\kappa) + 1}}} = O\left ( \sqrt{\dfrac{d^{\kappa}}{d^{D(\kappa)(1-\alpha) + (1--\alpha)}}}\right ).
    \end{equation}
    Since by definition $D(\kappa) = \lfloor \frac{\kappa}{1-\alpha}\rfloor$, we conclude that this term is also negligible. Hence, we have: 
    \begin{equation}
         \EE \left [ \max_{i \in n(d)}\left | \EE_{x} \left [ k_{d, >m(d)}(x,x')^2 \right ] - \EE_{x,x'} \left [k_{d, m(d)} (x,x')^2  \right ] \right | \right ]  = O(\mathrm{poly}\log(d) d^{-\frac{\delta_0}{2}})
    \end{equation}
    We conclude by Markov's inequality. 
\end{proof}

With this, we have proved Assumptions~\ref{assum:kernel_concentration}.

\subsection{Proof of \cref{assum:kernel_eigenvalue_decay}}

Assumption~\ref{assum:kernel_eigenvalue_decay} concerns properties about the eigenvalues. Recall we denote $m(d): =|\mathsf{Low}(n)|$, and $(\lambda_{d,i}, \psi_{i})_{i \geq 1}$ the eigen pairs of our kernel. We need to prove: 

\begin{enumerate}
    \item There exists $\delta_0 >0$, such that
            \[
            n(d)^{1 + \delta_0} \leq \dfrac{1}{\lambda_{d,m(d) + 1}^4} \sum_{k\geq m(d) + 1} \lambda_{d,k}^4,
            \]
            \[
            n(d)^{1 + \delta_0} \leq \dfrac{1}{\lambda_{d,m(d) + 1}^2} \sum_{k\geq m(d) + 1} \lambda_{d,k}^2.
            \]
   \item    \[
            m(d) \leq n(d)^{1- \delta_0}. 
            \]
\end{enumerate}

Recall that we already chose our value of $\delta_0$ in the definition of $\mathsf{High}(n)$ and $\mathsf{Low}(n)$, which we re-state now: 
\begin{align*}
    \mathsf{High}(n) & = \left \{ \beta \in \ZZ^{d}_{\geq 0}: |\beta|\leq L,  \sigma_1^{\beta_1} \cdots \sigma_d^{\beta_d} \leq \dfrac{1}{d^{\kappa + \delta_0}} \right \}\\ 
\mathsf{Low}(n) & = \left \{\beta \in \ZZ^{d}_{\geq 0}: \sigma_1^{\beta_1} \cdots \sigma_d^{\beta_d} > \dfrac{1}{d^{\kappa + \delta_0}} \right \}
\end{align*}   
Let's begin with the first part.

\begin{lemma} Consider the definitions of $\mathsf{High}(n)$ and $\mathsf{Low}(n)$ above. Then, there exists $\delta'_0$ such that:
\[
n(d)^{1 + \delta'_0} \leq \dfrac{1}{\lambda_{d,m(d) + 1}^4} \sum_{k\geq m(d) + 1} \lambda_{d,k}^4, 
\]
and 
\[
n(d)^{1 + \delta'_0} \leq \dfrac{1}{\lambda_{d,m(d) + 1}^2} \sum_{k\geq m(d) + 1} \lambda_{d,k}^2.
\]    
\end{lemma}

\begin{proof}
    We will only proof the second inequality. The first one will be analogous. Note that: 
    \begin{align}
       \lambda_{d,m(d) + 1}^2 & = C \max_{\beta \in \mathsf{High}(n)} \sigma_1^{\beta_1} \cdots \sigma_d^{\beta_d} \\
       & \leq \dfrac{C}{d^{\kappa + \delta_0}}. 
    \end{align}
    On the other hand: 
    \begin{align}
        \sum_{k\geq m(d) + 1} \lambda_{d,k}^2 = O(1),
    \end{align}
    as showed in \cref{lemma:properly_decaying_eigenvalues}. Then, for $d$ big enough, we have that:
    \begin{equation}
         \lambda_{d,m(d) + 1}^2 \leq \dfrac{1}{d^{\kappa + \delta_0}}  \sum_{k\geq m(d) + 1} \lambda_{d,k}^2,
    \end{equation}
    and re-writing this we get: 
    \begin{equation}
         d^{\kappa + \delta_0}  \leq \dfrac{1}{  \lambda_{d,m(d) + 1}^2}  \sum_{k\geq m(d) + 1} \lambda_{d,k}^2.
    \end{equation}
    Recalling that $n = Cd^{\kappa}$: 
    \begin{equation}
         n^{1 + \delta'_0} \leq \dfrac{1}{  \lambda_{d,m(d) + 1}^2}  \sum_{k\geq m(d) + 1} \lambda_{d,k}^2,
    \end{equation}
    and we conclude. 
\end{proof}

We are now left with proving that $m(d) \leq n(d)^{1- \delta_0}$. This has to do with the fact that the results in \cite{MMM22} require concentrating the feature matrix for the low order eigenfunctions, and for this, there has to be a gap between the number of samples and the number of concentrating features. The technique will be essentially the same we used to order eigenvalues in \cref{cor:power_law_decay_monomial}. 

\begin{lemma}
    Let $n = O(d^{\kappa + \delta_0})$, and assume $\kappa \not = \lfloor \kappa \rfloor$. Let $D(\kappa) = \lfloor \frac{\kappa}{1-\alpha} \rfloor$, and assume $D(\kappa)(1-\alpha) < \kappa$. Then, there exists a small $\delta'_0$ such that
    \[
    m(d) \leq n^{1- \delta_0},
    \]
    where $m(d) = |\mathsf{Low}(n)|$. 
\end{lemma}

\begin{proof}
    We will directly bound $m(d) = |\mathsf{Low}(n)|$. By definition, we have: 
    \begin{align}
        m(d) & =  |\mathsf{Low}(n)|  \\
        & = \left | \left \{\beta \in \ZZ^{d}_{\geq 0}: \sigma_1^{\beta_1} \cdots \sigma_d^{\beta_d} > \dfrac{1}{d^{\kappa + \delta_0}} \right \} \right |
        \label{eq:counting_eigenvalues_1}
    \end{align}
    As proved in \cref{lemma:concentration_F_1_F_2_low}, all $\beta \in \ZZ^{d}_{\geq 0}$ with $|\beta| \geq D(\kappa) + 1$ are in $\mathsf{High}(n)$. Therefore, $\mathsf{Low}(n) \subseteq \{\beta \in \ZZ^{d}_{\geq 0}: |\beta|\leq D(\kappa) \}$. With this, we can separate the cardinality in \cref{eq:counting_eigenvalues_1} according to the degree of $\beta$. We have:
    \begin{equation}
         m(d ) = \sum_{k=0}^{D(\kappa)}  \left | \left \{\beta \in \ZZ^{d}_{\geq 0}: |\beta| = k \text{ and } \sigma_1^{\beta_1} \cdots \sigma_d^{\beta_d} > \dfrac{1}{d^{\kappa + \delta_0}} \right \} \right |. 
    \end{equation}
    Also, note that the minimum possible eigenvalue that can be achieved by $\beta \in \ZZ^{d}_{\geq 0}$ with $|\beta| \leq \lfloor \kappa \rfloor$ is $d^{-\lfloor \kappa \rfloor}$. Therefore.
    \begin{align}
         m(d ) &= \sum_{k=0}^{\lfloor \kappa \rfloor}  \left | \left \{\beta \in \ZZ^{d}_{\geq 0}: |\beta| = k \text{ and } \sigma_1^{\beta_1} \cdots \sigma_d^{\beta_d} > \dfrac{1}{d^{\kappa + \delta_0}} \right \} \right | +  \sum_{k=\lfloor \kappa \rfloor +1}^{D(\kappa)}  \left | \left \{\beta \in \ZZ^{d}_{\geq 0}: |\beta| = k \text{ and } \sigma_1^{\beta_1} \cdots \sigma_d^{\beta_d} > \dfrac{1}{d^{\kappa + \delta_0}} \right \} \right | \\
         & = \sum_{k=0}^{\lfloor \kappa \rfloor}  \left | \left \{\beta \in \ZZ^{d}_{\geq 0}: |\beta| = k \right \} \right | +  \sum_{k=\lfloor \kappa \rfloor +1}^{D(\kappa)}  \left | \left \{\beta \in \ZZ^{d}_{\geq 0}: |\beta| = k \text{ and } \sigma_1^{\beta_1} \cdots \sigma_d^{\beta_d} > \dfrac{1}{d^{\kappa + \delta_0}} \right \} \right |.
    \end{align}
    We also now that
    \begin{equation}
        \left | \left \{\beta \in \ZZ^{d}_{\geq 0}: |\beta| = k \right \} \right | = \binom{d-1 + k}{d-1} = O(d^{k}). 
    \end{equation}
    Hence: 
    \begin{align}
         m(d ) & \leq Cd^{\lfloor \kappa \rfloor} + \sum_{k=\lfloor \kappa \rfloor +1}^{D(\kappa)}  \left | \left \{\beta \in \ZZ^{d}_{\geq 0}: |\beta| = k \text{ and } \sigma_1^{\beta_1} \cdots \sigma_d^{\beta_d} > \dfrac{1}{d^{\kappa + \delta_0}} \right \} \right |.
         \label{eq:cardinality_m_d_2}
    \end{align}
    By assumption, we have that $\kappa \not = \lfloor \kappa\rfloor $, we if we bound the cardinality of the RHS we can conclude. For this, we will proceed as we did in \cref{cor:power_law_decay_monomial}. Let $k \in \{ \lfloor \kappa \rfloor + 1, \dots, D(\kappa)$, and denote
    \begin{equation}
        M_{k} :=  \left | \left \{\beta \in \ZZ^{d}_{\geq 0}: |\beta| = k \text{ and } \sigma_1^{\beta_1} \cdots \sigma_d^{\beta_d} > \dfrac{1}{d^{\kappa + \delta_0}} \right \} \right |. 
    \end{equation}
    Then, by replace the definitions of $\sigma_j, j\in [d]$ we have: 
    \begin{align}
        M_k & = \left | \left \{\beta \in \ZZ^{d}_{\geq 0}: |\beta| = k \text{ and } \dfrac{\prod_{j=1}^d j^{-\alpha \beta_j}}{r_0(\Sigma)^{k}}> \dfrac{1}{d^{\kappa + \delta_0}} \right \} \right | \\
        & = \left | \left \{\beta \in \ZZ^{d}_{\geq 0}: |\beta| = k \text{ and } \prod_{j=1}^d j^{-\alpha \beta_j}> \dfrac{r_0(\Sigma)^{k}}{d^{\kappa + \delta_0}} \right \} \right | \\
        & =  \left | \left \{\beta \in \ZZ^{d}_{\geq 0}: |\beta| = k \text{ and } \prod_{j=1}^d j^{\alpha \beta_j} < \dfrac{d^{\kappa + \delta_0}}{r_0(\Sigma)^{k}} \right \} \right | \\
        & = \left | \left \{\beta \in \ZZ^{d}_{\geq 0}: |\beta| = k \text{ and } \prod_{j=1}^d j^{\beta_j} < \left (\dfrac{d^{\kappa + \delta_0}}{r_0(\Sigma)^{k}}\right)^{\frac{1}{\alpha}} \right \} \right |.
        \label{eq:midstep_bound_m}
    \end{align}
    We now identify that this is the same type of sets we saw in the proof of \cref{cor:power_law_decay_monomial}. Denote
    \begin{equation}
    X_k (L) : = \left | \left \{\beta \in \ZZ^{d}_{\geq 0}: |\beta| = k \text{ and } \prod_{j=1}^d j^{\beta_j} < L \right \} \right |
    \end{equation}
    Then, we can identify the cardinality of this set (via a bijection) with the cardinality of the set with: 
    \begin{equation}
        X_k (L) = \left | \left \{ (j_1, \dots, j_k): 1 \leq j_1 \leq \dots \leq j_k, \text{ and } \prod_{a=1}^k j_a  < L \right \} \right |.
    \end{equation}
    We can now apply the same technique we applied in \cref{cor:power_law_decay_monomial} (\cite{T15}, Chapter I.3), to get: 
    \begin{equation}
         X_k (L) = L \mathrm{poly} \log (L). 
    \end{equation}
    Then, going back to \cref{eq:midstep_bound_m}, we conclude that: 
    \begin{equation}
        M_k = \left (\dfrac{d^{\kappa + \delta_0}}{r_0(\Sigma)^{k}}\right)^{\frac{1}{\alpha}} \mathrm{poly}\log(d),
    \end{equation}
    and replacing this \cref{eq:cardinality_m_d_2}, we get: 
    \begin{align}
        m(d) &= O \left (d^{\lfloor \kappa \rfloor } + \sum_{k=\lfloor \kappa \rfloor +1}^{D(\kappa)} M_k \right ) \\
        & = O \left (d^{\lfloor \kappa \rfloor } + \sum_{k=\lfloor \kappa \rfloor +1}^{D(\kappa)} \left (\dfrac{d^{\kappa + \delta_0}}{r_0(\Sigma)^{k}}\right)^{\frac{1}{\alpha}} \mathrm{poly}\log(d)\right ) \\
        & =  O \left (d^{\lfloor \kappa \rfloor } + d^{\frac{\kappa + \delta_0}{\alpha}} \mathrm{poly}\log(d) \sum_{k=\lfloor \kappa \rfloor +1}^{D(\kappa)} \dfrac{1}{r_0(\Sigma)^{\frac{k}{\alpha}}}\right ).
    \end{align}
    The higher order term on the RHS corresponds to taking $k = \lfloor \kappa \rfloor + 1$. Then: 
    \begin{equation}
        m(d) = O \left ( d^{\lfloor \kappa \rfloor } + \dfrac{d^{\frac{\kappa + \delta_0}{\alpha}}}{r_0(\Sigma)^{\frac{\lfloor \kappa \rfloor + 1}{\alpha}}}\mathrm{poly}\log(d) \right ).
        \label{eq:mid_step_m_d_3}
    \end{equation}
    By \cref{eq:asymptotics_effective dimension_1}, we know that $r_0(\Sigma) = O(d^{1-\alpha})$. Then: 
    \begin{align}
        \dfrac{d^{\frac{\kappa + \delta_0}{\alpha}}}{r_0(\Sigma)^{\frac{\lfloor \kappa \rfloor + 1 + 1}{\alpha}}} & = O \left ( \dfrac{d^{\frac{\kappa + \delta_0}{\alpha}}}{d^{(1-\alpha)\frac{\lfloor \kappa \rfloor+ 1}{\alpha}}}\right ) \\
        & = O \left ( d^{\frac{\kappa -\lfloor \kappa \rfloor \cdot (1-\alpha) - (1-\alpha) + \delta_0}{\alpha} } \right ). 
    \end{align}
    By writing $\kappa= \alpha \kappa + (1-\alpha )\kappa$, we get: 
    \begin{align}
        \dfrac{d^{\frac{\kappa + \delta_0}{\alpha}}}{r_0(\Sigma)^{\frac{\lfloor \kappa \rfloor + 1 + 1}{\alpha}}}  & = O \left ( d^{\kappa + \frac{(1-\alpha) \kappa -\lfloor \kappa \rfloor \cdot (1-\alpha) - (1-\alpha) + \delta_0}{\alpha} } \right ) \\ 
        & = O \left ( d^{\kappa + \frac{(1-\alpha) (\kappa -\lfloor \kappa \rfloor  - 1) + \delta_0}{\alpha} } \right ). 
    \end{align}
    Then, since $\delta_0$ is very small, we conclude that there exists $\delta'_0$ such that: 
    \begin{equation}
         \dfrac{d^{\frac{\kappa + \delta_0}{\alpha}}}{r_0(\Sigma)^{\frac{\lfloor \kappa \rfloor + 1 + 1}{\alpha}}} \leq Cd^{\kappa - \delta'_0}. 
    \end{equation}
    Going back to \cref{eq:mid_step_m_d_3}, we get: 
    \begin{equation}
        m(d) \leq C \max \{d^{\lfloor \kappa \rfloor}, d^{\kappa - \delta'_0}\}  \ll n^{1- \delta'_0},
    \end{equation}
    so we conclude. 
\end{proof}

\end{document}